\documentclass{article}

\usepackage{microtype}
\usepackage{graphicx}
\usepackage{subcaption}
\usepackage{booktabs} % for professional tables

\usepackage{dblfloatfix}

\usepackage{placeins}

\usepackage{hyperref}

\usepackage{paralist} % Required in your preamble

\providecommand{\R}{\mathbb{R}}

\providecommand{\E}{\mathbb{E}}

\providecommand{\Var}{\operatorname{Var}}

\providecommand{\MST}{\operatorname{MST}}

\providecommand{\indep}{\mathrel{\perp\!\!\!\perp}}

\usepackage[preprint]{icml2026}

\usepackage{amsmath}
\usepackage{amssymb}
\usepackage{mathtools}
\usepackage{amsthm}

\usepackage[capitalize,noabbrev]{cleveref}

\theoremstyle{plain}
\newtheorem{theorem}{Theorem}[section]
\newtheorem{proposition}[theorem]{Proposition}
\newtheorem{lemma}[theorem]{Lemma}

\theoremstyle{definition}
\newtheorem{definition}[theorem]{Definition}
\newtheorem{assumption}[theorem]{Assumption}
\theoremstyle{remark}
\newtheorem{remark}[theorem]{Remark}

\usepackage[textsize=tiny]{todonotes}

\icmltitlerunning{Topological Residual Asymmetry for Bivariate Causal Direction}

\begin{document}

\twocolumn[
  \icmltitle{Topological Residual Asymmetry for Bivariate Causal Direction}



  \icmlsetsymbol{equal}{*}

  \begin{icmlauthorlist}
    \icmlauthor{Mouad El Bouchattaoui}{equal,comp}
    % \icmlauthor{Firstname2 Lastname2}{equal,yyy,comp}
    % \icmlauthor{Firstname3 Lastname3}{comp}
  \end{icmlauthorlist}

  \icmlaffiliation{comp}{Paris, France}

  \icmlcorrespondingauthor{Mouad El Bouchattaoui}{mouad.elbouchattaoui@gmail.com}

  \icmlkeywords{Machine Learning, ICML}

  \vskip 0.3in
]



\printAffiliationsAndNotice{}  % no special notice (required even if empty)

\begin{abstract} % ideally between 4--6 sentences long.
Inferring causal direction from purely observational bivariate data is fragile: many methods commit to a direction even in ambiguous or near non-identifiable regimes. We propose \emph{Topological Residual Asymmetry} (TRA), a geometry-based criterion for additive-noise models. TRA compares the shapes of two cross-fitted regressor--residual clouds after rank-based copula standardization: in the correct direction, residuals are approximately independent, producing a two-dimensional bulk, while in the reverse direction---especially under low noise---the cloud concentrates near a one-dimensional tube. We quantify this bulk--tube contrast using a 0D persistent-homology functional, computed efficiently from Euclidean MST edge-length profiles. We prove consistency in a triangular-array small-noise regime, extend the method to fixed noise via a binned variant (TRA-s), and introduce TRA-C, a confounding-aware abstention rule calibrated by a Gaussian-copula plug-in bootstrap. Extensive experiments across many challenging synthetic and real-data scenarios demonstrate the method's superiority.
\end{abstract}

\section{Introduction}
\label{sect:intro}

Inferring causal direction from observational data remains a central challenge in statistics and machine learning: for two real-valued
variables $(X,Y)$, dependence alone cannot distinguish $X\to Y$, $Y\to X$, or \emph{no orientation} without structural assumptions
\citep{SpirtesGlymourScheines2001, Pearl2009}. Identifiability is obtained by restricting the data-generating class, e.g. via
additive-noise models (ANMs), where the causal direction admits $Y=f(X)+\varepsilon$ with $\varepsilon\indep X$ while the reverse
representation typically fails \citep{Hoyer_ANM_2009, Peters_ANM_2014}, or via other asymmetry principles such as information-geometric
criteria linking the cause marginal to local expansion of the mechanism \citep{Janzing_IGCI_2012}. Most practical bivariate methods
instantiate these asymmetries by fitting both directions and comparing scalar scores, e.g., regress-then-test ANM pipelines (RESIT),
information-geometric scores (IGCI), regression-error asymmetries (RECI), and minimum description length (MDL)/regularized-regression proxies (SLOPPY) \citep{Mooij_CauseEffectPairs_JMLR_2016, Glymour_RevCD_2019}. However, such summaries can be brittle under finite samples or mild misspecification and often provide limited insight into \emph{why} a direction is preferred \citep{Mooij_CauseEffectPairs_JMLR_2016, Zhang_KCI_2012}; moreover, in near non-identifiable regimes (e.g., close to linear--Gaussian) they may still force a decision despite weak evidence rather than signaling ambiguity \citep{Hoyer_ANM_2009, Peters_ANM_2014}.

We take a complementary view: in ANMs, direction is encoded in the \emph{geometry} of the fitted regressor--residual cloud. In the causal direction, residuals are (asymptotically) independent of the regressor under consistent regression, as the ANM constraint requires \citep{Hoyer_ANM_2009, Peters_ANM_2014}. In the reverse direction, an independent-noise representation occurs only in special
non-identifiable cases (e.g., linear--Gaussian) \citep{Hoyer_ANM_2009, Peters_ANM_2014}, so residuals remain generically dependent and,
in low-noise injective settings, often nearly functional: the reverse cloud concentrates near a low-dimensional set (typically a smooth
curve) while the causal cloud retains a two-dimensional bulk. A copula transform \citep{Nelsen_CopulasIntro_2006} makes this contrast
invariant to monotone marginal reparameterizations.

We introduce \emph{Topological Residual Asymmetry} (TRA), a directional score that compares the shapes of the two fitted residual clouds. TRA first applies a rank-based copula standardization, making the comparison invariant to strictly increasing marginal transformations and therefore sensitive only to copula-level dependence. It then summarizes each cloud with a 0D persistent-homology functional \citep{Edelsbrunner_TopologicalPersistence_2002, Ghrist_barcode_2008}. For the standard union-of-balls (distance-to-set) filtration, the 0D merge scales coincide with the single-linkage hierarchy and can be read directly from Euclidean \emph{minimum-spanning-tree} (MST) edge lengths (up to $1/2$ radius factor) \citep{Gower_MSTSingleLinkage_1969}. These MST-based summaries serve as a practical proxy for intrinsic dimension via scaling laws for geometric graph functionals \citep{Steele1988, Penrose_RdGeomGrph_2003, CostaHero2004}, with proven stability under perturbations \citep{CohenSteiner_PersistenceStability_2007}.

\begin{figure}[!t]
  \centering
\includegraphics[width=0.9\linewidth,trim=0 10 0 10,clip]{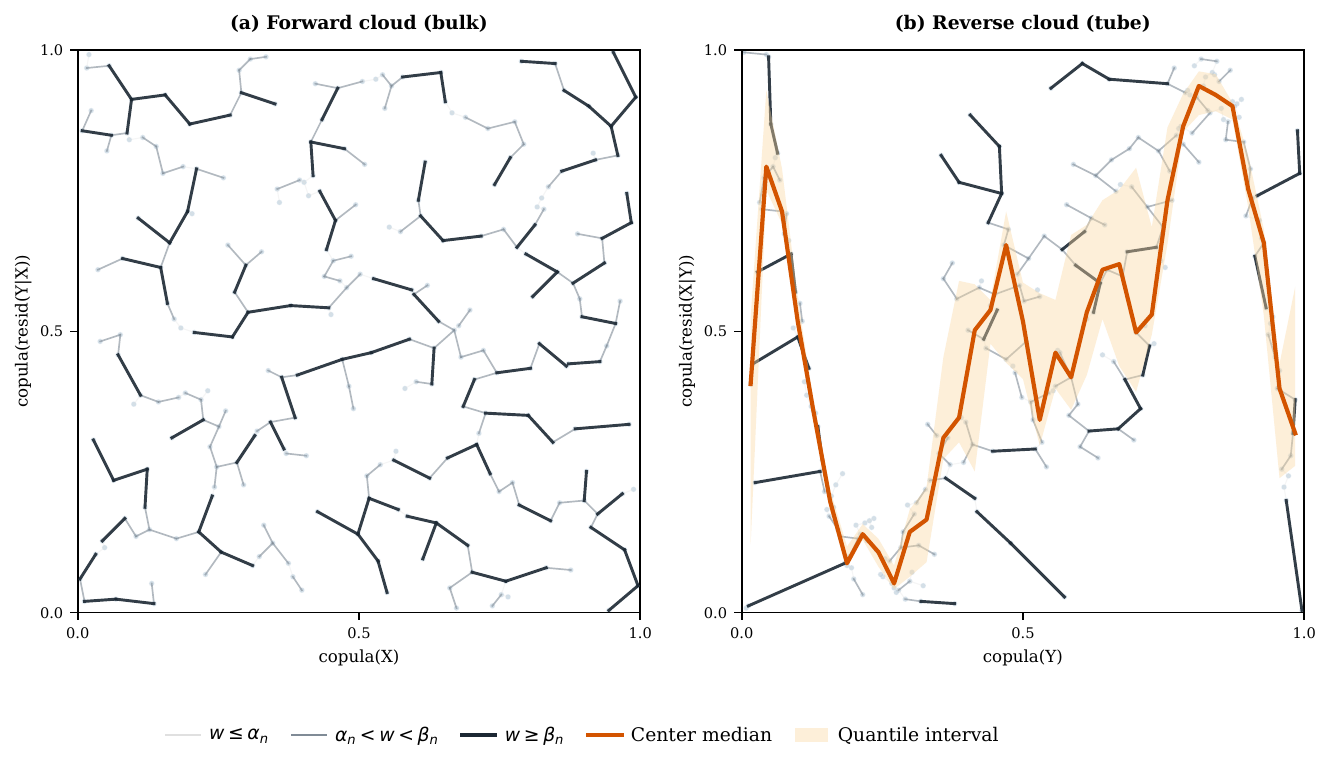}
  \caption{Bulk-tube contrast in copula transform. Forward clouds are 2D, while reverse clouds collapse near a curve. $Y_{n,i} \mid X_i \sim \mathcal{N}(X_i^3,n^{-1/2}), n=250$}. 
\label{fig:bulk_tube_smallnoise}
\end{figure}

Finally, we introduce \emph{TRA-C}, a confounding-aware version of TRA with abstention. Under a symmetric \emph{no-direction} null---dependence without causal asymmetry---the signed score concentrates near zero, so TRA-C abstains; when the observed asymmetry is atypical under the fitted null, TRA-C returns its sign. Implementations of all models, baselines, and experiments are provided in the supplementary material.

% advantage in experiments 
% Contributions 
\section{Related Work}
\label{sect:related_work}

Many bivariate orientation methods exploit an \emph{asymmetry principle}: a property that holds (approximately) in the true direction and fails except in special non-identifiable cases. They orient by comparing direction-specific scores, often favoring the direction that admits a \emph{simpler} or more \emph{generic} explanation. IGCI targets near-deterministic mechanisms and, under an independence-of-cause-and-mechanism view, links the cause density to local expansion of the transfer function \citep{Janzing_IGCI_2012}. RECI compares directions via regression error \citep{Bloebaum_RECI_2018}, while MDL-inspired methods such as SLOPPY use description-length surrogates from local/global regression fits \citep{MarxVreeken_SLOPPY_2019}. Kernel and distributional criteria such as KCDC and CDCI contrast direction-specific conditional-distribution complexity \citep{Mitrovic_KCDC_2018, Duong_CDCI_2022}, and quantile diagnostics such as bQCD compare multiple conditional quantiles to capture changing conditional shape, including heteroscedasticity \citep{Tagasovska_bQCD_2020}. TRA fits within this family but differs in the object it compares: rather than a single slope, error, or divergence statistic, it contrasts the \emph{multiscale geometry of residual point clouds}. Moreover, TRA operates on \emph{copula-standardized} residual clouds, so the relevant geometry reflects the \emph{dependence structure} (copula geometry) rather than marginal scale.

A key identifiable model class is the ANM. Under suitable regularity conditions, this identifies the causal direction except in known non-identifiable settings (notably some linear--Gaussian families) \citep{Hoyer_ANM_2009, Peters_ANM_2014}. In practice, ANM methods regress in both directions and select the one in which the residuals are most consistent with independence from the regressor, using dependence measures such as HSIC or conditional-independence tests such as KCI \citep{Gretton_HSIC_2005, Zhang_KCI_2012}. RESIT extends this regress-then-test idea to multivariate DAG learning; in the bivariate case, it reduces to an ANM-style orientation rule \citep{Peters_ANM_2014}. TRA departs from this template in what it extracts from residuals: rather than scoring residual (in)dependence directly, it summarizes the \emph{geometry of the residual point cloud}.

TRA’s functional is closely related to intrinsic-dimension and manifold-diagnostics methods, where nearest-neighbor graphs and MSTs yield statistics for detecting low-dimensional structure or estimating dimension via scaling laws \citep{CostaHero_GMST_2004, Camastra_IDsurvey_2016}. After copula standardization, TRA exploits this multiscale spacing to contrast ``bulk-like'' and ``curve-like'' residual clouds. Rather than estimating a global intrinsic dimension for exploration, TRA uses an MST--persistence functional as a \emph{directional asymmetry diagnostic} for causal orientation, with an explicit abstention option when separation is unreliable.

A separate line of work treats bivariate orientation as supervised learning on labeled cause--effect pairs. RCC trains a classifier on
randomized feature maps of empirical distributions \citep{LopezPaz_RCC_2015}, while NCC learns end-to-end representations, typically from synthetic and/or curated pairs \citep{LopezPaz_NCC_2017}; support-measure machines similarly embed each dataset as a probability measure and apply kernels on measures \citep{Varando_SMM_2024}. In parallel, generative and compression-based methods compare
direction-specific explanations, including CGNN-style generative modeling \citep{Goudet_CGNN_2018}, nonlinear-ICA ensembles such as Causal Mosaic \citep{WuFukumizu_CausalMosaic_2020}, and variational Bayesian compression \citep{Tran_VBCompression_2025}.

\section{Theoretical Setup}
For this section, we defer detailed proofs of all theoretical claims to Appendix~\ref{appendix:proofs}. All notions from metric geometry and geometric graphs used in the proofs are recalled in the Background (Appendix~\ref{appendix:background}). We work in the bivariate setting and, for each \(n\), observe \(n\) \emph{row-wise i.i.d.} samples \(\mathcal D_n := \{(X_i,Y_{n,i})\}_{i=1}^n\). To cover both fixed- and vanishing-noise regimes, we consider the additive-noise model \citep{Hoyer_ANM_2009, Peters_ANM_2014}:
\begin{equation}
\label{eq:anm_triangular}
\resizebox{0.8\columnwidth}{!}{$
Y_{n,i}=f(X_i)+\sigma_n\varepsilon_i,\;
\varepsilon_i\!\indep\!X_i,\;
\mathbb{E}[\varepsilon_i]=0,\;
\sigma_n\ge0
$}
\end{equation}
Here \((X_i,\varepsilon_i)_{i=1}^n\) are i.i.d., while \(\sigma_n\) may depend on \(n\). The \emph{fixed-noise} regime has \(\sigma_n\equiv\sigma>0\), whereas the \emph{small-noise} regime assumes \(\sigma_n\downarrow 0\) as \(n\to\infty\). In the forward (putatively causal) direction, the regression function is \(f\). In the reverse direction, set \(m_n(y):=\mathbb E[X\mid Y_n=y]\), which may vary with \(n\) through \(\sigma_n\).

Fix an integer \(K\ge 2\) and a deterministic \(K\)-fold partition \(\{I_k\}_{k=1}^K\) of \([n]\) \citep{Chernozhukov_DML_2018}; write \(k(i)\) for the unique index such that \(i\in I_{k(i)}\). For each fold \(k\), fit forward and backward regressors \(\widehat f_n^{(-k)}\) and \(\widehat g_n^{(-k)}\) on the training set \(\{(X_j,Y_{n,j})\}_{j\notin I_k}\). Define the out-of-fold residuals \(r_i^{(Y\mid X)}:=Y_{n,i}-\widehat f_n^{(-k(i))}(X_i)\) and \(r_i^{(X\mid Y)}:=X_i-\widehat g_n^{(-k(i))}(Y_{n,i})\). Their oracle counterparts replace the estimators by the population targets \(f\) and \(m_n\): \(r_{i,\circ}^{(Y\mid X)}:=Y_{n,i}-f(X_i)\) and \(r_{i,\circ}^{(X\mid Y)}:=X_i-m_n(Y_{n,i})\). For any direction \(A|B\), write \(\mathcal R^{(n)}_{A\mid B}:=\{(B_i,r_i^{(A\mid B)})\}_{i=1}^n\), and define \(\mathcal R^{(n)}_{\circ,A\mid B}\) analogously using oracle residuals. In particular, \(\mathcal R^{(n)}_{Y\mid X}=\{(X_i,r_i^{(Y\mid X)})\}_{i=1}^n\) and \(\mathcal R^{(n)}_{X\mid Y}=\{(Y_{n,i},r_i^{(X\mid Y)})\}_{i=1}^n\), similarly for \(\mathcal R^{(n)}_{\circ,Y\mid X}\) and \(\mathcal R^{(n)}_{\circ,X\mid Y}\).

\paragraph{Copula standardization.}
Our goal is to compare the \emph{dependence geometry} of the forward and reverse residual clouds. Doing so directly in \(\mathbb{R}^2\) is often fragile: marginal features---scale, skewness, heavy tails, or heteroscedasticity---can dominate Euclidean geometry even when the underlying dependence is unchanged. Copula coordinates mitigate this nuisance by applying the marginal CDF to each coordinate \citep{Nelsen_CopulasIntro_2006}.

For a pair \((U,V)\) with continuous marginals, define \(T_{U,V}(u,v):=(F_U(u),F_V(v))\in [0,1]^2\).
In the oracle construction we copula-standardize each direction separately:  $\widetilde{\mathcal R}^{(n)}_{\circ,Y\mid X} := T_{X,r^{(Y\mid X)}_\circ}(\mathcal R^{(n)}_{\circ,Y\mid X})$ and $\widetilde{\mathcal R}^{(n)}_{\circ,X\mid Y} := T_{Y_n,r^{(X\mid Y)}_\circ}(\mathcal R^{(n)}_{\circ,X\mid Y})$. In practice, the marginal CDFs are unknown, so we replace them by ranks \citep{Fermanian_EmpricalCopula_2004}. For a univariate sample \((U_i)_{i=1}^n\), define \(T_n(U_i):=\mathrm{rank}(U_i)/(n+1)\). Applying \(T_n\) coordinatewise gives the empirical copula-standardized clouds
\[
\widetilde{\mathcal R}^{(n)}_{Y\mid X}:=T_n\big(\mathcal R^{(n)}_{Y\mid X}\big),
\qquad
\widetilde{\mathcal R}^{(n)}_{X\mid Y}:=T_n\big(\mathcal R^{(n)}_{X\mid Y}\big).
\]

We view a cloud \(\mathcal R=\{\mathbf Z_1,\dots,\mathbf Z_n\}\subset\mathbb R^2\) as a finite metric space with distance
\(\|\cdot\|_2\). The \emph{Euclidean minimum spanning tree} (MST) \(\MST(\mathcal R)\) is the graph on vertex set \(\mathcal R\) that
connects all points with \(n-1\) edges and minimizes the total edge weight, where the weight of an edge \(e=(\mathbf Z_i,\mathbf Z_j)\) is its Euclidean length \(w(e):=\|\mathbf Z_i-\mathbf Z_j\|_2\) \citep{Kruskal_MST_1956,Gower_MSTSingleLinkage_1969}. We write the multiset of MST edge weights as
\[
\{\ell_j(\mathcal R)\}_{j=1}^{n-1}:=\{w(e): e\in\MST(\mathcal R)\}. 
\]

\paragraph{Mesoscopic scale: separating ``bulk'' from ``curve''.}
After copula standardization, each residual cloud forms a point set $\mathcal R$. We define the local separation scale as the typical interpoint distance, which we can represent by the nearest-neighbor (NN) distances
\[
d_i \;:=\; \min_{j\neq i}\|\mathbf Z_i-\mathbf Z_j\|_2.
\]
A comparable notion of local scale is provided by \emph{typical} Euclidean MST edge lengths on $\mathcal R$ (the MST contains many edges at nearest-neighbor scale, though it may also include longer bridging edges).

This local separation reflects the \emph{intrinsic dimension} of the support. For a genuinely $d$-dimensional cloud with a density bounded away from $0$ and $\infty$ on a set of positive $d$-volume, typical NN distances scale as $\Theta(n^{-1/d})$ (e.g.\ in probability), hence as $\Theta(n^{-1/2})$ in $d=2$ \citep{PenroseYukich_NN_2009}. Moreover, for i.i.d.\ samples in $\mathbb R^d$, the total Euclidean MST length scales as $\Theta(n^{(d-1)/d})$, so the \emph{average} MST edge length is $\Theta(n^{-1/d})$ \citep{SeoYukich_MST_2000}. In contrast, if the cloud is supported on a one-dimensional $C^1$ curve and sampled i.i.d.\ with a density bounded away from $0$ along arclength, local separations scale as $\Theta(n^{-1})$, and the maximal spacing along the curve is $\Theta((\log n)/n)$ almost surely \citep{Devroye_UniformSpacings_1981}. These two scales motivate a mesoscopic window $\alpha_n$ lying strictly between $n^{-1/2}$ and $n^{-1}$, separating ``bulk-like'' from ``curve-like'' geometry:
\begin{equation}\label{eq:mesoscopic_conditions}
\frac{\log n}{n}=o(\alpha_n),
\quad
n^{-1}\ll \alpha_n \ll n^{-1/2},
\quad
\beta_n \asymp \alpha_n .
\end{equation}
We use the concrete choice
\begin{equation}\label{eq:mesoscopic_window}
\alpha_n=\kappa n^{-2/3},
\quad
\beta_n=c_\beta\,\alpha_n,
\quad
\kappa>0,\ c_\beta>1,
\end{equation}
which satisfies \eqref{eq:mesoscopic_conditions}. At this scale, bulk clouds still look ``thick'' with many MST edges exceeding \(\alpha_n\), whereas curve-like clouds already look ``tight'' with most edges falling below \(\alpha_n\). To measure how much connectivity happens \emph{inside} \([\alpha,\beta]\), define the soft window
\begin{equation}\label{eq:psi_window}
\Psi_{\alpha,\beta}(t):=(\min\{t,\beta\}-\alpha)_+\in[0,\beta-\alpha],
\end{equation}
and, for a cloud \(\mathcal R\) of size \(n\ge 2\), the score
{\small
\begin{equation}\label{eq:TP_profile}
\overline{\mathrm{TP}}^{[\alpha,\beta]}_0(\mathcal R)
:= \frac{1}{(n-1)(\beta-\alpha)}
\sum_{e\in\MST(\mathcal R)}
\Psi_{\alpha,\beta}\big(w(e)\big).
\end{equation}
}
where $w(e)=\|\mathbf Z_i-\mathbf Z_j\|_2$ denotes the Euclidean length of the edge $e=(\mathbf Z_i,\mathbf Z_j)$. Large values of \(\overline{\mathrm{TP}}^{[\alpha,\beta]}_0\) mean many MST edges lie above \(\alpha\) (bulk-like at that scale); values near \(0\) mean most merges occur below \(\alpha\) (tube-like at that scale).

\begin{definition}[TRA score and direction]\label{def:tra_score}
Define the normalized TRA score
{\small
\begin{equation}\label{eq:Delta_emp_main}
\Delta_n :=
\overline{\mathrm{TP}}^{[\alpha_n,\beta_n]}_0(\widetilde{\mathcal R}^{(n)}_{Y\mid X})
-\overline{\mathrm{TP}}^{[\alpha_n,\beta_n]}_0(\widetilde{\mathcal R}^{(n)}_{X\mid Y})
\in[-1,1].
\end{equation}
}
We predict \(X\to Y\) if \(\Delta_n>0\) and \(Y\to X\) if \(\Delta_n<0\).
\end{definition}

\paragraph{Core intuition.}
In an ANM, the forward residual is (approximately) independent of the cause \citep{Peters_ANM_2014}, and the standardized forward cloud behaves like a genuinely two-dimensional scatter in \([0,1]^2\).
In the reverse direction, in low-noise injective settings the reverse residual becomes nearly deterministic in \(Y\), so the standardized reverse cloud concentrates around a curve (a thin tube). TRA detects this contrast at the intermediate window \([\alpha_n,\beta_n]\): bulk clouds yield larger \(\overline{\mathrm{TP}}^{[\alpha_n,\beta_n]}_0\) than tube-like clouds, hence the sign of \(\Delta_n\) reveals the direction.

The bulk--tube heuristic becomes a theorem once we control two effects: (i) the small-noise geometry of the \emph{oracle} copula clouds; and (ii) perturbations from cross-fitted regression and rank--copula standardization at the mesoscopic scale $\alpha_n$. We package the required conditions in three assumptions.

\begin{assumption}[Forward Model regularity]\label{ass:A_model}
\(X\) is supported on a compact interval \(I=[a,b]\), and \(f:I\to\R\) is \(C^1\).
There exist \(J\ge 1\) and breakpoints \(a=t_0<t_1<\cdots<t_J=b\) such that on each \(I_j:=(t_{j-1},t_j)\),
\(f\) is strictly monotone and
\[
0<c_f \le |f'(x)|\le C_f \qquad (x\in I_j,\ j=1,\dots,J).
\]
Let \(\mathcal V:=\{f(t_1),\dots,f(t_{J-1})\}\) be the turning-value set. The noise \(\varepsilon\) has a continuous density and is sub-Gaussian, hence \(\max_{1\le i\le n}|\varepsilon_i|=O_{\mathbb P}(\sqrt{\log n})\).
\end{assumption}

\begin{assumption}[Reverse target regularity]\label{ass:A_reverse_reg}
Fix \(\eta\in[0,1]\). There exists a compact interval \(J_\eta\Subset f(I)\setminus\mathcal V\) such that
\(\mathbb P(Y_n\in J_\eta)\to 1\), the density \(p_{Y_n}\) is uniformly bounded below on \(J_\eta\), and \(m_n\) is uniformly continuous on \(J_\eta\) uniformly in \(n\).
\end{assumption}

\begin{assumption}[Estimation error]\label{ass:A_scale_est}
\textnormal{(i)} \emph{Forward $L^2$-risk consistency:} $\E\|\widehat f_n^{(-k)}-f\|_{L^2(\mathbb P_X)}^2\to 0$ for each $k$. \\
\textnormal{(ii)} \emph{Reverse uniform-on-test-points consistency on $J_\eta$:}
for each $k$,
$\delta_n^{(k)}:=
\max_{i\in I_k:\,Y_{n,i}\in J_\eta}
|\widehat g_n^{(-k)}(Y_{n,i})-m_n(Y_{n,i})|
=o_{\mathbb P}(\alpha_n)$. \\
\textnormal{(iii)} \emph{Noise thickness is negligible:}
$\sigma_n\sqrt{\log n}=o(\alpha_n)$.
\end{assumption}

\paragraph{Remarks.}
Assumption~\ref{ass:A_model} permits global non-invertibility but requires branchwise invertibility with Lipschitz inverse constant
$1/c_f$, which controls the thickness of the reverse-direction ``tube.'' Assumption~\ref{ass:A_reverse_reg} excludes neighborhoods of
$\mathcal V$, ensuring the inverse branches are well behaved on $J_\eta$. Assumption~\ref{ass:A_scale_est} places the scale window
$[\alpha_n,\beta_n]$ between the connectivity regimes of the curve and the bulk, and requires both regression error and noise-induced
thickness to be $o_{\mathbb P}(\alpha_n)$. Together, Assumptions~\ref{ass:A_model}--\ref{ass:A_scale_est} mirror smoothness, tail, and
sample-splitting conditions common in ANM identifiability and cross-fitted nuisance estimation; see, e.g.,
\citet{Hoyer_ANM_2009, Peters_ANM_2014, Chernozhukov_DML_2018,Vershynin2018HDP}.

\begin{theorem}[Small-noise TRA consistency (normalized score)]\label{thm:small_noise_consistency}
Assume the small-noise ANM $Y_n=f(X)+\sigma_n\varepsilon$ with $\sigma_n\downarrow 0$ and Assumptions~\ref{ass:A_model}--\ref{ass:A_scale_est}. Then 
\[
\overline{\mathrm{TP}}^{[\alpha_n,\beta_n]}_0\!\big(\widetilde{\mathcal R}^{(n)}_{Y\mid X}\big)\xrightarrow{\mathbb P}1,
\quad
\overline{\mathrm{TP}}^{[\alpha_n,\beta_n]}_0\!\big(\widetilde{\mathcal R}^{(n)}_{X\mid Y}\big)\xrightarrow{\mathbb P}0.
\]
Hence \(\Delta_n\xrightarrow{\mathbb P}1\). In particular, for any deterministic threshold $\tau_n\downarrow 0$, the abstaining rule $\widehat{\mathrm{dir}}_n=X\to Y$ if $\Delta_n>\tau_n$, $Y\to X$ if $\Delta_n<-\tau_n$, and \emph{abstain} otherwise, satisfies
\[
\mathbb P(\widehat{\mathrm{dir}}_n=X\to Y)\to 1,
\qquad
\mathbb P(\widehat{\mathrm{dir}}_n=\emph{abstain})\to 0.
\]
\end{theorem}

\subsection{Fixed Noise Regime}
\label{subsect:fixe_noise_regime}
In the small-noise regime $\sigma_n\downarrow 0$, the reverse copula cloud becomes one-dimensional at the mesoscopic scale: under
branchwise invertibility of $f$, the reverse oracle cloud concentrates within an $o(\alpha_n)$-tube around a compact
$1$-dimensional set $\Gamma_{n,\eta}^{\cup}$, a finite union of $C^1$ curves (one per monotone branch). Consequently, all reverse MST
merges fall below $\alpha_n$. However, with fixed noise,
\[
Y=f(X)+\varepsilon,\qquad \varepsilon\indep X,\qquad \Var(\varepsilon)=\sigma^2>0,
\]
this collapse disappears. Even with the oracle reverse target $m(y):=\E[X\mid Y=y]$, the reverse fluctuation $\xi:=X-m(Y)$ satisfies
$\E[\xi\mid Y]=0$ but typically $\Var(\xi\mid Y)=\Theta(1)$. After copula standardization the reverse cloud has $O(1)$ thickness, remains bulk-like, and $\overline{\mathrm{TP}}^{[\alpha_n,\beta_n]}_0$ need not drift to $0$.

A smoothed variant, \emph{TRA-s}, recovers a one-dimensional signature by bin-averaging the reverse residuals along the \(Y\)-copula coordinate \citep{Tukey_regressogram_1961,Cattaneo_Binscatter_2024}. Let \(U_i:=T_n(Y_{n,i})\in[0,1]\), and partition \([0,1]\) into \(B_n\) equal bins \(I_{n,b}:=((b-1)/B_n,b/B_n]\). Write \(J_b:=\{i:U_i\in I_{n,b}\}\) and \(N_b:=|J_b|\). For each nonempty bin, define
\[
\bar u_b:=\frac1{N_b}\sum_{i\in J_b}U_i,
\qquad
\bar r_b:=\frac1{N_b}\sum_{i\in J_b}\bigl(X_i-\widehat g_n^{(-k(i))}(Y_{n,i})\bigr),
\]
and the binned reverse cloud \(\widehat{\mathcal R}^{(n)}_{X\mid Y}:=\{(\bar u_b,\bar r_b):N_b\ge 1\}\). If \(\widehat g_n=m\), then \(\bar r_b\) averages mean-zero fluctuations and typically satisfies \(|\bar r_b|=O_{\mathbb P}(N_b^{-1/2})\), so \(\widehat{\mathcal R}^{(n)}_{X\mid Y}\) concentrates near \(\Gamma:=\{(u,0):u\in[0,1]\}\), up to discretization error \(O(B_n^{-1})\) and regression error. Choosing \(B_n\) so these terms are \(o(\widetilde\alpha_n)\) makes the binned cloud tube-like at the mesoscopic scale.

\vspace{-0.5em}
\paragraph{Assumptions for TRA-s under fixed noise.}
TRA-s uses a partitioning smoother (regressogram/binscatter) on the \(Y\)-copula axis \citep{Tukey_regressogram_1961,Cattaneo_Binscatter_2024}. We impose a standard shrinking-bin regime to control bin means uniformly \citep{GyorfiEtAl2002NonparametricRegression}, a conditional sub-Gaussian assumption for reverse fluctuations \citep{Vershynin2018HDP}, and cross-fitting to decouple nuisance estimation from evaluation \citep{Chernozhukov_DML_2018}.

\begin{assumption}[Binning growth rates]\label{ass:binning_scheme}
With \(U_i\), \(I_{n,b}\), \(J_b\), \(N_b\), and \(m_n:=|\{b:N_b\ge 1\}|\) as above, assume
\[
B_n\to\infty,\qquad B_n=o(n),\qquad B_n^{7/3}\frac{\log B_n}{n}\to 0 .
\]
\end{assumption}

\begin{assumption}[Conditional reverse fluctuations]\label{ass:reverse_noise_subG}
Let \(m(y):=\E[X\mid Y=y]\) and \(\xi:=X-m(Y)\). Assume \(\|\xi\mid Y=y\|_{\psi_2}\le K_0\sigma\) for all \(y\) in the support of \(Y\).
\end{assumption}

\begin{assumption}[Bin-averaged reverse regression error]\label{ass:bin_reverse_reg_err}
Let \(\widehat g_n^{(-k)}\) be cross-fitted and let \(\bar r_b\) be defined as above for bins with \(N_b\ge 1\). Set
\(\widetilde\alpha_n:=\kappa m_n^{-2/3}\) and \(\widetilde\beta_n:=c_\beta\,\widetilde\alpha_n\). Assume
\[
\max_{b:\,N_b\ge 1}\Bigg|
\frac1{N_b}\sum_{i\in J_b}\bigl(\widehat g_n^{(-k(i))}(Y_{n,i})-m(Y_{n,i})\bigr)
\Bigg| = o_{\mathbb P}(\widetilde\alpha_n).
\]
\end{assumption}

\begin{theorem}[Fixed-noise TRA-s consistency]\label{thm:fixed_noise_TRAs_compact}
Under fixed noise, assume Assumptions~\ref{ass:A_model}--\ref{ass:A_scale_est} and
\ref{ass:binning_scheme}--\ref{ass:bin_reverse_reg_err}. Let \(\alpha_n=\kappa n^{-2/3}\), \(\beta_n=c_\beta\alpha_n\), let \(m_n\)
be the number of nonempty bins, and set \(\widetilde\alpha_n=\kappa m_n^{-2/3}\), \(\widetilde\beta_n=c_\beta\widetilde\alpha_n\).
With \(\widehat{\mathcal R}^{(n)}_{X\mid Y}\) the binned reverse cloud, define
\[
\widetilde\Delta_{0,n}
:=\overline{\mathrm{TP}}^{[\alpha_n,\beta_n]}_0\!\big(\widetilde{\mathcal R}^{(n)}_{Y\mid X}\big)
-\overline{\mathrm{TP}}^{[\widetilde\alpha_n,\widetilde\beta_n]}_0\!\big(\widehat{\mathcal R}^{(n)}_{X\mid Y}\big)\in[-1,1].
\]
Then \(\overline{\mathrm{TP}}^{[\alpha_n,\beta_n]}_0\!\big(\widetilde{\mathcal R}^{(n)}_{Y\mid X}\big)\xrightarrow{\mathbb P}1\) and \(\overline{\mathrm{TP}}^{[\widetilde\alpha_n,\widetilde\beta_n]}_0\!\big(\widehat{\mathcal R}^{(n)}_{X\mid Y}\big)\xrightarrow{\mathbb P}0.\)

Hence \(\Delta_{0,n}\xrightarrow{\mathbb P}1\) and  \(\E[\widetilde\Delta_{0,n}]\to 1\).
\end{theorem}

\subsection{TRA-C: abstention under unobserved confounding via a Gaussian-copula plug-in bootstrap}
\label{subsect:TRAC}

\begin{figure}[t]
  \centering
  \includegraphics[width=.8\linewidth]{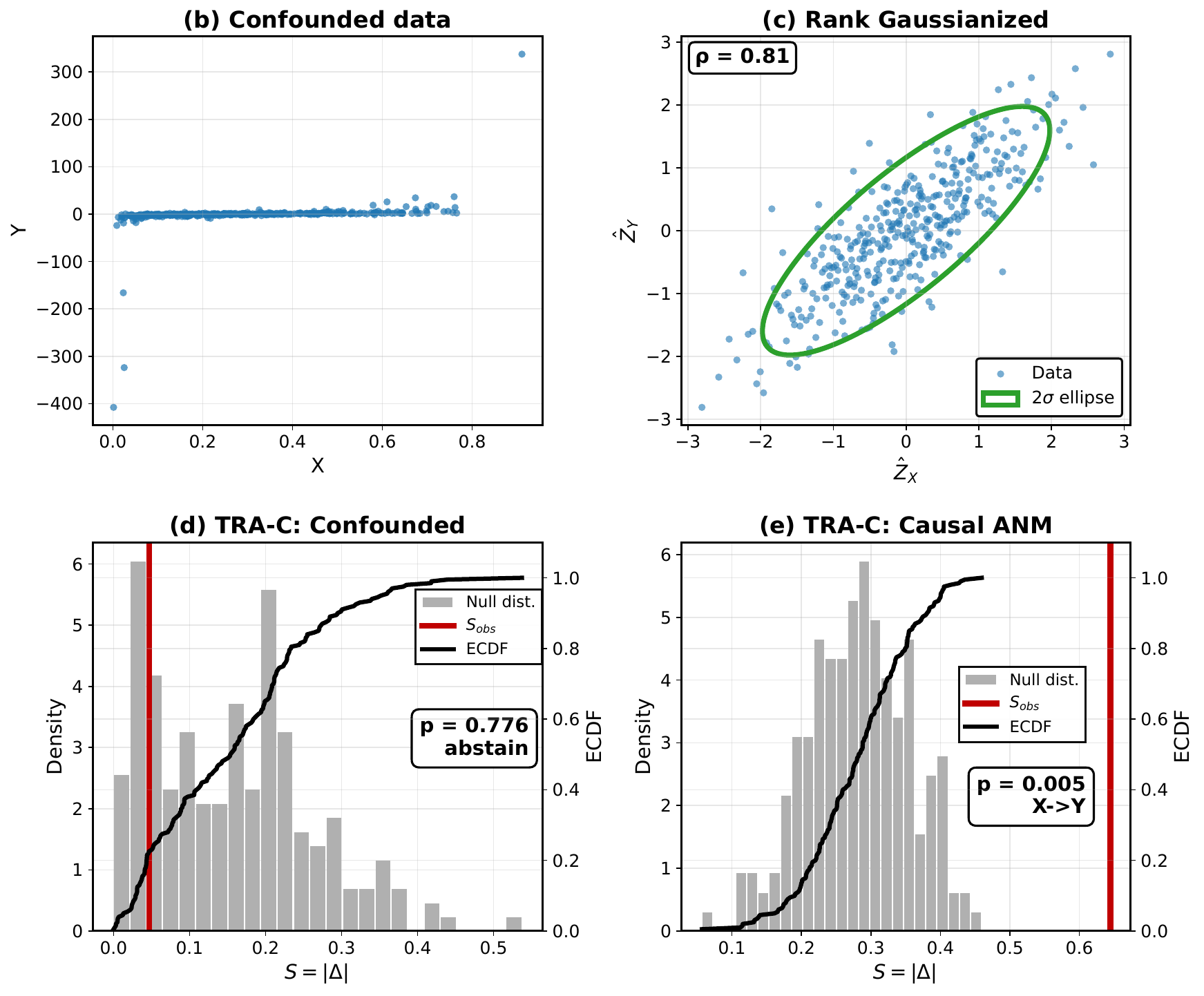}
  \caption{\textbf{TRA-C calibration on a Gaussian-copula confounding null.} Rank-Gaussianized data fit a correlation $\widehat\rho$ and induce a Gaussian-copula bootstrap for the score $\Delta_n$.}
  \label{fig:trac_confounding_example}
\end{figure}

Theorems~\ref{thm:fixed_noise_TRAs_compact} and \ref{thm:small_noise_consistency} analyze the base direction scores \(\Delta_n\) (TRA) and \(\widetilde\Delta_{0,n}\) (TRA-s) under \emph{unconfounded} ANMs. In observational studies, however, the observed dependence may be induced by an \emph{unobserved confounder}, in which case any forced orientation is not causally meaningful \citep{Pearl2009,SpirtesGlymourScheines2001}. We therefore introduce TRA-C, which equips TRA with an \emph{abstain} option by calibrating the score magnitude against a fitted \emph{confounding-only} null, in the spirit of decision rules with reject options \citep{Chow1970,ElYanivWiener2010}. Specifically, we test for confounding using the two-sided magnitude \(S_n:=|\Delta_n|\) for TRA and \(S_n:=|\widetilde\Delta_{0,n}|\) for TRA-s. The calibration is otherwise identical; only the definition of \(S_n\) differs.

As a toy confounding model, let $H\sim\mathcal N(0,1)$, $Z_X=aH+\xi_X$, $Z_Y=bH+\xi_Y$, with $(\xi_X,\xi_Y)\indep H$ and $(\xi_X,\xi_Y)\sim\mathcal N\!\bigl(0,\mathrm{diag}(\sigma_X^2,\sigma_Y^2)\bigr)$, followed by unknown strictly increasing marginals $X=F_X^{-1}(\Phi(Z_X))$, $Y=F_Y^{-1}(\Phi(Z_Y))$. Figure~\ref{fig:trac_confounding_example} illustrates the logic of TRA-C: after rank Gaussianization, $(\widehat Z_X,\widehat Z_Y):=(\Phi^{-1}(\mathrm{rank}(X)/(n+1)),\,\Phi^{-1}(\mathrm{rank}(Y)/(n+1)))$ is well approximated by a bivariate normal with $\widehat\rho\approx 0.81$, yielding an elliptical Gaussian-copula null (shown by the $2\sigma$ contour). TRA-C then calibrates the score magnitude $S_n:=|\mathrm{Score}_n|$ against this null: under confounding (panel~(d)) $S_{\mathrm{obs}}$ is typical ($p\approx 0.776$) and TRA-C abstains, while under a genuine ANM (panel~(e)) $S_{\mathrm{obs}}$ is in the tail ($p\approx 0.005$) and TRA-C returns a direction.

\paragraph{Gaussian-copula null and Gaussianization.}
As a working ``no-direction'' dependence baseline, we take \(\mathcal M_0\) to be the bivariate Gaussian-copula family: distributions with arbitrary continuous marginals \((F_X,F_Y)\) whose copula is Gaussian, i.e., after marginal Gaussianization \(Z_X:=\Phi^{-1}(F_X(X))\) and \(Z_Y:=\Phi^{-1}(F_Y(Y))\), the pair \((Z_X, Z_Y)\) is bivariate normal with correlation \(\rho\in(-1,1)\) \citep{Nelsen_CopulasIntro_2006}. This removes marginal scales and isolates dependence; under \(P\in\mathcal M_0\) the \emph{copula} is parameterized by \(\rho\) alone, so a plug-in \(\hat\rho_n\) yields a symmetric baseline for \(S_n\) \citep{LiuLaffertyWasserman2009,Hoff2007}.

\paragraph{TRA-C calibration.}
We calibrate \(S_n:=|\Delta_n|\) (or \(S_n:=|\widetilde\Delta_{0,n}|\)) under a fitted Gaussian-copula null
\(\widehat M_{0,n}=M_{\hat\rho_n}\), where \(\hat\rho_n\) is the correlation of the rank-Gaussianized marginals.
Algorithm~\ref{alg:trac} details the procedure: TRA-C reports a direction only when \(S_n\) is extreme under
\(\widehat M_{0,n}\), and otherwise abstains. This plug-in parametric-bootstrap calibration in a semiparametric
copula model is supported by classical goodness-of-fit validity results \citep{GenestRemillard2008}.
\begin{algorithm}[tb]
\caption{TRA-C: Gaussian-copula calibration with abstention}
\label{alg:trac}
\small
\begin{algorithmic}[1]
\STATE \textbf{Input:} $\mathcal D_n=\{(X_i,Y_i)\}_{i=1}^n$, signed score $\Delta=\mathsf{Score}(\cdot)$, $B$, $\alpha$.
\STATE Compute $\Delta_n\leftarrow \mathsf{Score}(\mathcal D_n)$, $S_n\leftarrow|\Delta_n|$.
\STATE Rank-Gaussianize and fit copula correlation:
$\hat Z_{X,i}\!\leftarrow\!\Phi^{-1}(\mathrm{rank}(X_i)/(n{+}1))$,
$\hat Z_{Y,i}\!\leftarrow\!\Phi^{-1}(\mathrm{rank}(Y_i)/(n{+}1))$,
$\hat\rho_n\leftarrow \mathrm{Corr}(\hat Z_X,\hat Z_Y)$.
\FOR{$b=1,\dots,B$}
  \STATE Draw $(Z_X^*,Z_Y^*)\stackrel{\text{i.i.d.}}{\sim}\mathcal N(0,\Sigma(\hat\rho_n))$; set $(U_X^*,U_Y^*)=\Phi(Z_X^*,Z_Y^*)$.
  \STATE $(X^*,Y^*)\leftarrow(\hat F_X^{-1}(U_X^*),\hat F_Y^{-1}(U_Y^*))$; $S_n^{*(b)}\leftarrow\bigl|\mathsf{Score}(\mathcal D_n^{*(b)})\bigr|$.
\ENDFOR
\STATE $\hat p_{n,B}\leftarrow \dfrac{1+\sum_{b=1}^B \mathbf 1\{S_n^{*(b)}\ge S_n\}}{B+1}$.
\STATE \textbf{return} $\;\textsf{abstain}$ if $\hat p_{n,B}>\alpha$, else $\textbf{sign}(\Delta_n)\in\{X\!\to\!Y,\;Y\!\to\!X,\;\textsf{abstain}\}$.
\end{algorithmic}
\end{algorithm}

\begin{theorem}[TRA-C: asymptotic level under a plug-in Gaussian-copula bootstrap]
\label{thm:TRAC_level_rigorous}
Fix \(\alpha\in(0,1)\). Let \(\widehat\rho_n=\widehat\rho_n(\mathcal D_n)\) be measurable and define the fitted copula null \(\widehat M_{0,n}:=M_{\widehat\rho_n}\). Let \(\mathrm{Score}_n\in\{\Delta_n,\widetilde\Delta_{0,n}\}\) and set \(S_n:=|\mathrm{Score}_n|\). Conditionally on \(\mathcal D_n\), draw a bootstrap sample from \(M_{\widehat\rho_n}\) and compute its analogue \(S_n^*\). Define the conditional bootstrap CDF \(\widehat F_n^*(t):=\Pr_{\widehat\rho_n}\!\big(S_n^*\le t \,\big|\, \mathcal D_n\big)\), \(t\in\mathbb R\). Let \(B_n\to\infty\) and let \(S_n^{*(1)},\dots,S_n^{*(B)}\) be i.i.d.\ conditionally on \(\mathcal D_n\) with conditional CDF
\(\widehat F_n^*\). Define the (conservative) Monte-Carlo \(p\)-value
\[
\widehat p_{n,B_n}
:=\frac{1+\sum_{b=1}^{B_n} \mathbf 1\{S_n^{*(b)}\ge S_n\}}{B_n+1},
\]
and the TRA-C rejection (non-abstention) event \(R_n:=\{\widehat p_{n,B_n}\le \alpha\}\).
Assume that under \(P_{\rho_0}\in\mathcal M_0\):
\begin{inparaenum}
    \item \(\widehat\rho_n \to \rho_0\) in \(P_{\rho_0}\)-probability;
\item with \(F_{\rho_0,n}(t):=\Pr_{\rho_0}(S_n\le t)\),
\[
\sup_{t\in\mathbb R}\big|\widehat F_n^*(t)-F_{\rho_0,n}(t)\big|\xrightarrow{P_{\rho_0}}0;
\]
\item \(B_n\to\infty\).
\end{inparaenum}
Then the TRA-C test has asymptotic level at most \(\alpha\):
\[
\limsup_{n\to\infty}\Pr_{\rho_0}(R_n)\le \alpha.
\]
\end{theorem}

The guarantee is \emph{relative to the working null} $\mathcal M_0$. If the data lie in the Gaussian-copula null, then TRA-C outputs a direction with asymptotic probability at most $\alpha$. We do not claim that all confounded distributions lie in $\mathcal M_0$; rather, $\mathcal M_0$ is a broad ``dependence-without-direction'' family (e.g., latent Gaussian factor structure with unknown monotone marginals), and the plug-in bootstrap prevents directional claims when this null fits.

\section{Experiments}
\label{sect:exps}
\paragraph{Baselines and decision rule.}
We compare unsupervised, per-pair methods (TRA/TRA-s, RESIT, IGCI, RECI, CDCI, and the per-pair neural compression baseline COMIC) and supervised predictors (RCC, NCC) trained on independent synthetic labeled pairs and evaluated on all T\"ubingen pairs. For any signed score $\Delta_n$, we permit abstention via the symmetric reject rule: output $X\!\to\!Y$ if $\Delta_n>\tau_n$, $Y\!\to\!X$ if $\Delta_n<-\tau_n$, and abstain otherwise. We set $\tau_n$ in \emph{stability} mode (no labels): draw $R$ subsamples (fraction $0.8$) of the observed pair, compute $\Delta_n^{(r)}$ on each, and define
\[
\tau_n := z_{1-\alpha/2}\,\widehat{\mathrm{sd}}\!\big(\Delta_n^{(1)},\ldots,\Delta_n^{(R)}\big),\quad R=50.
\]
We run TRA-s and TRA-C with a Gaussian-copula confounding null, using $B=500$ null bootstraps and $\alpha=0.10$ for the two-sided significance/abstention test.

\vspace{-0.5em}
\paragraph{Metrics.}
We summarize performance by \emph{coverage} $\mathrm{Cov}=\frac{n_{\mathrm{decided}}}{n}$), \emph{decided accuracy} $\mathrm{Acc}_d=\frac{n_{\mathrm{correct}}}{n_{\mathrm{decided}}}$, and the single scalar \emph{directional risk} $\mathrm{Risk}=(1-\mathrm{Acc}_d)\mathrm{Cov}=\frac{n_{\mathrm{wrong}}}{n}$, so a method is penalized either for wrong calls at high coverage or for excessive abstention, with Wilson-score confidence intervals for accuracy over decided pairs.

\vspace{-0.5em}
\paragraph{Baseline hyperparameter selection.}
To match our setting---unsupervised orientation from a single observational pair; we perform no label-based tuning on evaluation pairs. For unsupervised baselines (ANM/RESIT variants, IGCI, RECI, CDCI, bQCD/QCCD, COMIC, SLOPPY), we use authors’ defaults or \emph{label-free} selection applied symmetrically in both directions: cross-validation for regression/conditional models, standard heuristics for
independence tests (e.g., HSIC median bandwidth), and objective-based selection when available (e.g., COMIC ELBO). Any abstention thresholds (RECI/SLOPPY confidence rules, ANM test rejection, CDCI non-causal output) are fixed \emph{a priori} or null calibrated, never tuned to ground-truth directions. Supervised transfer baselines (RCC, NCC) are trained/tuned only on external labeled/synthetic data and then frozen. For methods that permit a choice of regression backbone (e.g. TRA/TRA-s, RESIT, RECI), we standardize it to the same nonparametric regressor, smoothing splines \citep{Wahba_1990_Splines}.

\subsection{Synthetic Data}
\label{subsect:exp_synthetic}

\paragraph{DGPs.}
We consider four synthetic additive-noise regimes, each with ground-truth direction $X\to Y$. In each regime we vary a single \emph{stress} parameter and hold all others fixed:
\begin{inparaenum}[(i)]
    \item \emph{Cubic ANM (SNR sweep):} $X\sim\mathcal N(0,1)$, $Y=X^3+\sigma_\varepsilon\epsilon$, $\epsilon\sim\mathcal N(0,1)$; sweep $\sigma_\varepsilon$ from low to high noise.
    \item \emph{Near-linear ANM (nonlinearity sweep):} $X\sim\mathcal N(0,1)$, $Y=X+cX^3+\sigma_\varepsilon\epsilon$, $\epsilon\sim\mathcal N(0,1)$ with $\sigma_\varepsilon=0.3$; sweep $c\downarrow 0$ toward the near-nonidentifiable linear limit.
    \item \emph{Heteroscedastic cubic ANM (variance-drift sweep):} $X\sim\mathcal N(0,1)$, $Y=X^3+\epsilon$ with $\epsilon=(\sigma_0+\lambda|X|)\xi$, $\xi\sim\mathcal N(0,1)$ and $\sigma_0=0.3$; sweep $\lambda\ge 0$ to increase heteroscedasticity.
    \item \emph{Non-monotone sine ANM (noise sweep):} $X\sim\mathrm{Unif}[-1,1]$, $Y=\sin(X)+\sigma_\varepsilon\epsilon$, $\epsilon\sim\mathcal N(0,1)$; sweep $\sigma_\varepsilon$ to increase noise under a non-monotone mechanism.
\end{inparaenum}
For each scenario, we sweep sample size $n\in\{50,100,150,250,500,1000,1500,2000\}$ and run $n_{\mathrm{rep}}=30$ independent Monte Carlo replicates per $(n,\text{parameter})$ setting. In each replicate we draw one dataset $\{(X_i,Y_i)\}_{i=1}^n$ and evaluate all methods on the same draw.

Figure~\ref{fig:synthetic_anm_directed_risk_atlas} reports \emph{directed risk} for all baselines across four ANM stress tests with
ground-truth $X\to Y$, as we vary sample size $n$ (vertical axis) and a scenario-specific stress parameter (horizontal axis). Most baselines are strongly \emph{scenario dependent}: they excel in one panel yet fail in another, producing the high-risk bands across the atlas. In contrast, the TRA family is consistently strong: raw TRA is best in the small-noise corner, while TRA-s stays low-risk in fixed-noise regimes where raw TRA cannot exploit reverse-cloud collapse. This matches the theory: Theorem~\ref{thm:small_noise_consistency} makes $\Delta_n$ informative as $\sigma_n\downarrow 0$, and Theorem~\ref{thm:fixed_noise_TRAs_compact} shows that smoothing ($\widetilde\Delta_{0,n}$) restores separation at fixed noise by averaging reverse fluctuations at a mesoscopic scale. 

In \emph{cubic ANM (noise/SNR sweep)}, raw TRA improves rapidly as noise decreases, in line with Theorem~\ref{thm:small_noise_consistency}: in the low-noise corner the reverse copula cloud becomes tube-like and MST merges fall below $\alpha_n$, so $\Delta_n$ cleanly separates directions. As noise grows this collapse disappears and raw TRA degrades, whereas TRA-s stays stable across the sweep, consistent with Theorem~\ref{thm:fixed_noise_TRAs_compact}. Several classical pipelines fail in stress-dependent bands, and supervised predictors (NCC/RCC/COMIC) show pronounced distribution-shift sensitivity with broad high-risk regions. In \emph{near-linear ANM (nonlinearity sweep)}, identifiability weakens as the cubic perturbation vanishes. As expected, many baselines
develop elevated risk bands as the decision signal collapses. Even in this corridor, TRA-s remains among the most reliable
methods, highlighting that it is driven by forward--reverse residual \emph{geometry} after copula standardization and smoothing, rather
than by a single parametric asymmetry that must stay bounded away from zero. In \emph{heteroscedastic cubic ANM (variance-drift sweep)}, increasing $\lambda$ violates homoscedastic-noise assumptions and degrades
several baselines. Methods that exploit richer conditional-distribution information (e.g., CDCI/bQCD) are more tolerant once $n$ is
large enough to estimate these quantities. TRA-s remains low-risk across $\lambda$ and $n$, consistent with mesoscopic bin-averaging
damping variance-driven fluctuations while preserving directional separation. Finally, in the \emph{non-monotone sine ANM (noise sweep)}, the forward map is many-to-one, so reverse regression must mix multiple
inverse branches. Supervised predictors again show pronounced instability across the grid, whereas TRA-s remains competitive and improves with $n$. In the \emph{monotone} regimes (cubic and near-linear), the strong performance of TRA-s also supports the role of ourforward-model regularity condition (Assumption~\ref{ass:A_model}), under which the reverse oracle cloud concentrates on a finite union of smooth strands as a noise vanishes.

\vspace{-0.5em}
\paragraph{Latent confounding setting (no true direction).}
To evaluate abstention under dependence-without-direction, we simulate \emph{confounding-only} models driven by a latent $Z$. In the linear case, $Z\sim\mathcal N(0,1)$, $X=Z+\varepsilon_X$, $Y=\gamma Z+\varepsilon_Y$, with $\varepsilon_X\sim\mathcal N(0,\sigma_X^2)$
and $\varepsilon_Y\sim\mathcal N(0,\sigma_Y^2)$ independent of $Z$; we sweep $\gamma$. We also test a \emph{nonlinear} confounded
scenario by adding a symmetric cubic warp: $Z\sim\mathcal N(0,1)$, $X=Z+aZ^3+\varepsilon_X$, and $Y=bZ+aZ^3+\varepsilon_Y$, with $b$
fixed and $a$ swept. Since neither $X\!\to\!Y$ nor $Y\!\to\!X$ holds, \emph{directional accuracy is undefined} and we report
\emph{coverage} (the fraction of non-abstained decisions). Figure~\ref{fig:coverage_confounded} tracks coverage (with confidence
intervals) as a function of the stress parameter ($\gamma$ or $a$; top row) and, after averaging over the stress sweep, as a function of
sample size $n$ (bottom row). Most direction-forcing baselines (RESIT, IGCI, CDCI, RCC, bQCD, SLOPPY) sustain near-unit coverage across parameters, i.e., they almost always output a direction despite the absence of a causal arrow. A partial exception is RECI: its mean coverage drops with $n$ in parts of the sweep, plausibly because its score difference shrinks or destabilizes under symmetric confounding; but this abstention is
uncalibrated and offers no controlled ``no-direction'' guarantee. In contrast, TRA-C calibrates score magnitude with a plug-in
Gaussian-copula bootstrap under a confounding null and exhibits the intended reject behavior, abstaining rather than hallucinating an
arrow in both the linear and nonlinear sweeps.

\begin{figure}[H]
  \centering
  \includegraphics[width=.5\textwidth]{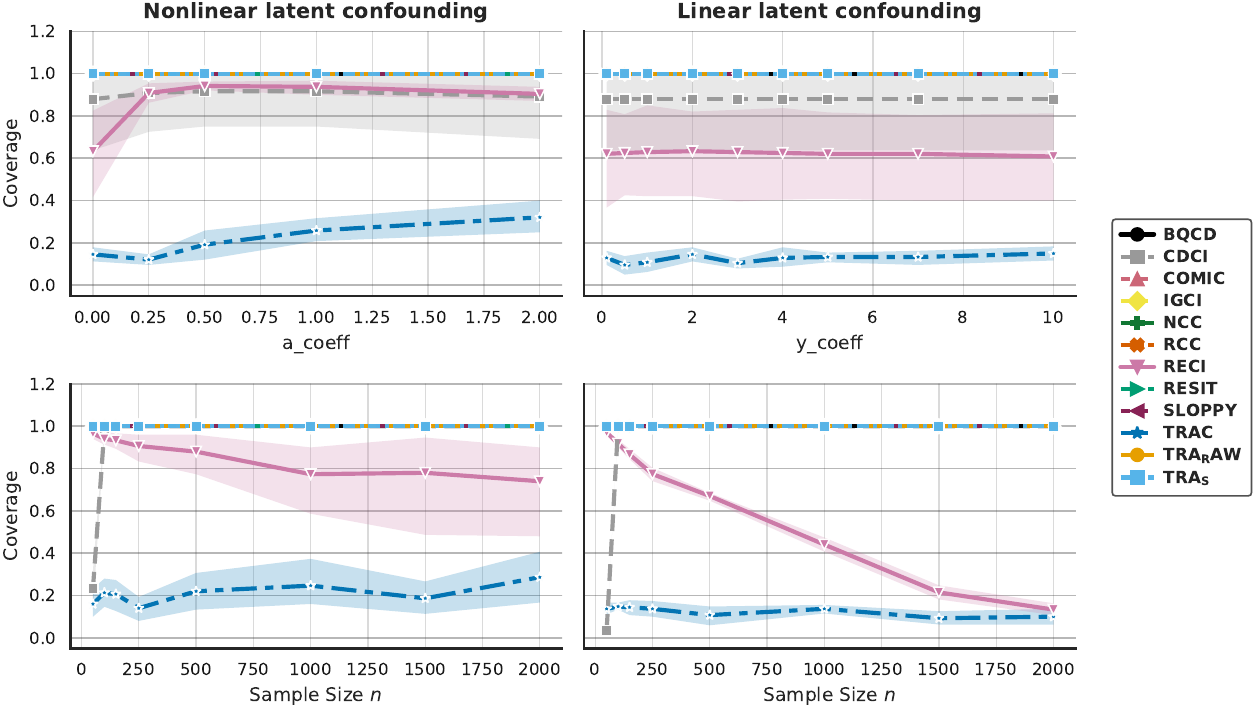}
  \caption{Evolution of coverage under (non)linear latent confounding as function of stress parameters (top) and sample size (bottom). Lower is better.}
  \label{fig:coverage_confounded}
\end{figure}

\begin{figure*}[!t]
  \centering
  \includegraphics[width=\textwidth]{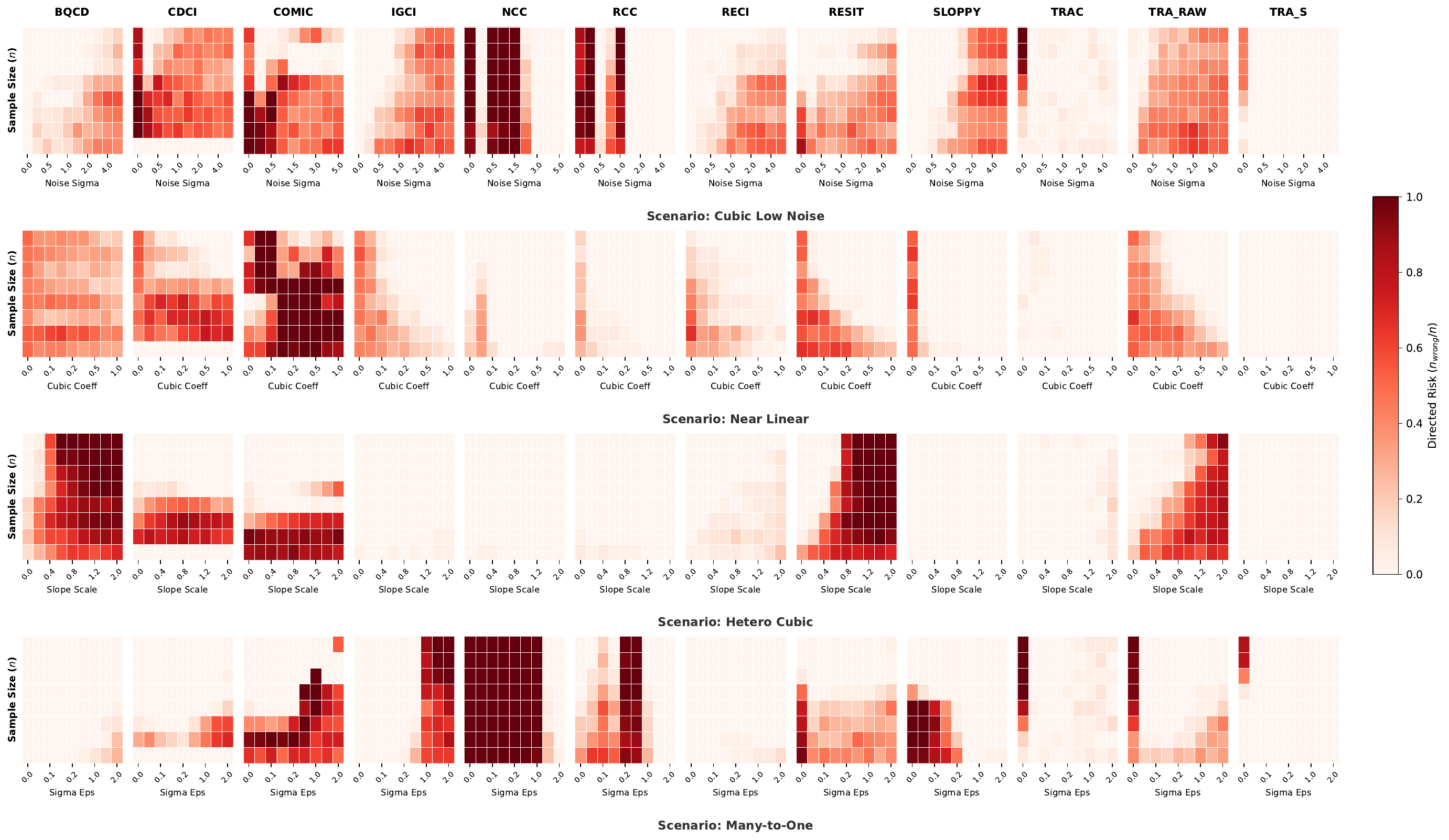}
    \caption{\textbf{Synthetic ANM atlas.} Directed risk  vs.\ sample size $n$ and the scenario stress parameter. TRA-s is consistently low-risk across regimes; baselines show regime-specific failures.}
  \label{fig:synthetic_anm_directed_risk_atlas}
\end{figure*}

\begin{figure*}[!t]
  \centering
  \includegraphics[width=.8\textwidth]{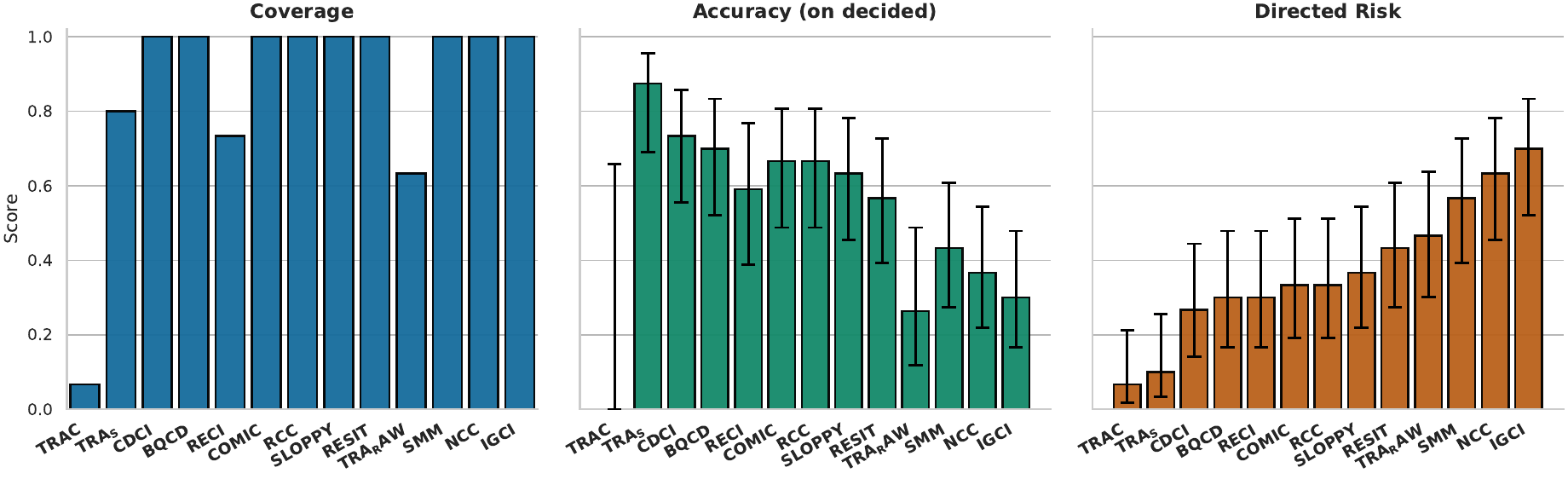}
  \caption{\textbf{Results on T\"ubingen cause--effect pairs.} Coverage, accuracy conditional on deciding, and directed risk for all methods.}
  \label{fig:tubingen_summary_stats}
\end{figure*}
\subsection{Real-world experiments}
\label{subsect:exp_tubingen}
We benchmark on curated observational cause--effect datasets with unknown data-generating processes, where standard identifiability
assumptions may fail, using the T\"ubingen Cause--Effect Pairs benchmark \citep{Mooij_CauseEffectPairs_JMLR_2016} (108 labeled real-world pairs). For compatibility with common univariate baselines, we restrict to 30 pairs where both the annotated cause and effect are
univariate. Figure~\ref{fig:tubingen_summary_stats} summarizes performance via \emph{coverage}, \emph{accuracy on decided}, and
\emph{directed risk} (penalizing both wrong calls and abstention), yielding two conclusions. 

First, TRA variants give the best accuracy--coverage tradeoff: TRA-C is deliberately conservative, attaining the highest decided accuracy and lowest directed risk overall, while TRA-s is the strongest high-coverage option, deciding on most pairs with (best/near-best) accuracy and the second-lowest risk. Second, competing families show the expected distribution-shift brittleness on real pairs:
scalar-score baselines (RECI, RESIT, SLOPPY) and distributional criteria (CDCI, bQCD) achieve only moderate accuracy and higher risk than TRA-s, and supervised predictors are least reliable (NCC and IGCI incur the largest risks), consistent with sensitivity to heterogeneous real-world mechanisms and marginals. Finally, TRA-C's low risk is not due to abstention alone---risk penalizes abstention---but to a reduction in wrong orientations that outweighs the abstention penalty; when coverage is the priority, TRA-s is the appropriate operating point.

\textbf{Conclusion.} We introduced \emph{Topological Residual Asymmetry} (TRA), which infers causal direction from the copula-standardized geometry of cross-fitted regressor--residual clouds via Euclidean MST summary. We proved consistency in a triangular-array
small-noise regime, showed that mesoscopic binning (TRA-s) restores separation under fixed noise, and proposed TRA-C, a
confounding-aware abstention rule calibrated by a Gaussian-copula null. Across synthetic stress tests and T\"ubingen pairs, TRA
variants deliver the best accuracy--coverage tradeoff and lowest directed risk. Next steps include extending TRA to multivariate discovery by using local residual-asymmetry tests for edge orientation \citep{Peters_CausalInference_2017, Zheng_NO_TEARS_2018}, and leveraging limit theory for random geometric graphs/complexes to obtain calibrated uncertainty and higher-order topological summaries beyond MSTs \citep{Penrose_RdGeomGrph_2003, BobrowskiKahle_RandomGeometricComplexesSurvey_2014, SkrabaYogeshwaran_MinSpanningAcycles_2022}.

\section*{Impact Statement}
This work advances methodology for causal direction inference from observational data by proposing a topological, abstention-capable criterion for bivariate additive-noise models. Potential positive impacts include more reliable orientation decisions, especially via abstention in ambiguous or confounded settings, which may reduce downstream harm from overconfident causal claims in scientific applications. At the same time, causal direction estimates can be misused or overinterpreted, particularly when applied outside the assumed regime or in high-stakes domains; our results therefore do not justify automated decision-making without domain expertise and careful validation. We expect no direct negative societal impact from releasing this method and provide abstention, calibration mechanisms, and confidence intervals intended to discourage unwarranted causal conclusions. 

\bibliography{refs}
\bibliographystyle{icml2026}

%%%%%%%%%%%%%%%%%%%%%%%%%%%%%%%%%%%%%%%%%%%%%%%%%%%%%%%%%%%%%%%%%%%%%%%%%%%%%%%
%%%%%%%%%%%%%%%%%%%%%%%%%%%%%%%%%%%%%%%%%%%%%%%%%%%%%%%%%%%%%%%%%%%%%%%%%%%%%%%
% APPENDIX
%%%%%%%%%%%%%%%%%%%%%%%%%%%%%%%%%%%%%%%%%%%%%%%%%%%%%%%%%%%%%%%%%%%%%%%%%%%%%%%
%%%%%%%%%%%%%%%%%%%%%%%%%%%%%%%%%%%%%%%%%%%%%%%%%%%%%%%%%%%%%%%%%%%%%%%%%%%%%%%
\newpage
\appendix
\onecolumn

\section{Background}
\label{appendix:background}

This appendix collects standard probabilistic and geometric notions used in the proofs.

\subsection{Probability, laws, and stochastic order}
\label{subsec:probability_background}

\paragraph{Underlying probability space.}
All random objects are defined on a common probability space $(\Omega,\mathcal F,\mathbb P)$.
For a measurable map $Z:\Omega\to\mathbb R^2$, we write $\mathcal L(Z)$ for its distribution.

\begin{definition}[Law / pushforward]
\label{def:law_pushforward}
Let $Z:\Omega\to\mathbb R^2$ be measurable. The \emph{law} of $Z$ is the pushforward probability measure
\[
\mathcal L(Z)\ :=\ \mathbb P\circ Z^{-1}\ \in\ \mathcal P(\mathbb R^2).
\]
Equivalently, for every bounded measurable $\varphi:\mathbb R^2\to\mathbb R$,
\[
\int_{\mathbb R^2}\varphi(z)\,\mathcal L(Z)(dz)\ =\ \mathbb E[\varphi(Z)].
\]
\end{definition}

\begin{definition}[Coupling]
\label{def:coupling}
Let $\mu,\nu\in\mathcal P(\mathbb R^2)$. A \emph{coupling} of $(\mu,\nu)$ is a probability measure
$\pi\in\mathcal P(\mathbb R^2\times\mathbb R^2)$ with marginals $\mu$ and $\nu$, i.e.\ for all Borel sets
$A\subset\mathbb R^2$,
\[
\pi(A\times\mathbb R^2)=\mu(A),
\qquad
\pi(\mathbb R^2\times A)=\nu(A).
\]
Equivalently, a coupling can be realized as a pair of random vectors $(U,V)$ on a single probability space such that
$\mathcal L(U)=\mu$, $\mathcal L(V)=\nu$, and then $\pi=\mathcal L(U,V)$.
\end{definition}

\paragraph{Stochastic order notation.}
For sequences of random variables $(Z_n)$ and positive reals $(a_n)$, we use $O_{\mathbb P}$ and $o_{\mathbb P}$ in the standard sense.

\begin{definition}[Big-$O$ and little-$o$ in probability]
\label{def:Op_op}
We write $Z_n=O_{\mathbb P}(a_n)$ if $(|Z_n|/a_n)$ is bounded in probability, i.e.\ for every $\varepsilon>0$ there exist
$M<\infty$ and $n_0$ such that $\mathbb P(|Z_n|>Ma_n)\le \varepsilon$ for all $n\ge n_0$.
We write $Z_n=o_{\mathbb P}(a_n)$ if $Z_n/a_n\xrightarrow{\mathbb P}0$, i.e.\ for every $\varepsilon>0$,
$\mathbb P(|Z_n|>\varepsilon a_n)\to 0$.
\end{definition}

\begin{remark}[Basic algebra]
If $Z_n=o_{\mathbb P}(a_n)$ then $Z_n=O_{\mathbb P}(a_n)$.
If $Z_n=O_{\mathbb P}(a_n)$ and $W_n=O_{\mathbb P}(b_n)$, then
$Z_n+W_n=O_{\mathbb P}(a_n+b_n)$ and $Z_nW_n=O_{\mathbb P}(a_nb_n)$.
\end{remark}

\subsection{Bounded--Lipschitz distance}
\label{subsec:BL_metric}

We measure weak convergence of probability measures on $\mathbb R^2$ using the bounded--Lipschitz metric.

\begin{definition}[Bounded--Lipschitz metric]
\label{def:bounded_lip_metric}
For $\mu,\nu\in\mathcal P(\mathbb R^2)$ and measurable $\varphi:\mathbb R^2\to\mathbb R$, set
$\|\varphi\|_\infty:=\sup_{z\in\mathbb R^2}|\varphi(z)|$ and
$\mathrm{Lip}(\varphi):=\sup_{z\neq z'}|\varphi(z)-\varphi(z')|/\|z-z'\|_2$.
Let
\[
\mathrm{BL}_1(\mathbb R^2)
:=\{\varphi:\|\varphi\|_\infty\le 1,\ \mathrm{Lip}(\varphi)\le 1\}.
\]
Define
\[
d_{\mathrm{BL}}(\mu,\nu)
:= \sup_{\varphi\in \mathrm{BL}_1(\mathbb R^2)}
\Bigl|\int \varphi\,d\mu-\int \varphi\,d\nu\Bigr|.
\]
\end{definition}

\begin{remark}[Coupling bound and finiteness]
\label{rem:dBL_coupling}
For any coupling $(U,V)$ of $(\mu,\nu)$ and any $\varphi\in \mathrm{BL}_1(\mathbb R^2)$,
$|\mathbb E[\varphi(U)]-\mathbb E[\varphi(V)]|\le \mathbb E\|U-V\|_2$, hence
\[
d_{\mathrm{BL}}(\mu,\nu)\ \le\ \mathbb E\|U-V\|_2 .
\]
Unlike $W_1$, $d_{\mathrm{BL}}$ is finite for all probability measures since test functions are uniformly bounded.
\end{remark}

\subsection{Copulas and empirical copula coordinates}
\label{subsec:copulas}

\paragraph{Copulas.}
Let $(X,Y)$ have joint CDF $F_{X,Y}$ and marginal CDFs $F_X,F_Y$.

\begin{definition}[Bivariate copula]
\label{def:copula}
A \emph{copula} is a CDF $C:[0,1]^2\to[0,1]$ with uniform marginals.
Equivalently, $C$ is the joint CDF of some $(U,V)$ with $U,V\sim\mathrm{Unif}[0,1]$.
\end{definition}

\begin{theorem}[Sklar]
\label{thm:sklar}
There exists a copula $C$ such that $F_{X,Y}(x,y)=C(F_X(x),F_Y(y))$ for all $x,y$.
If $F_X$ and $F_Y$ are continuous, then $C$ is unique and $(F_X(X),F_Y(Y))$ has copula $C$.
\end{theorem}

\paragraph{Pseudo-observations via ranks.}
Given a univariate sample $W_1,\dots,W_n$ with continuous distribution, define
\[
\mathrm{rank}(W_i):=\#\{j:W_j\le W_i\}\in\{1,\dots,n\},
\qquad
\widehat U_i:=\frac{\mathrm{rank}(W_i)}{n+1}\in\Bigl(\frac{1}{n+1},\frac{n}{n+1}\Bigr),
\]
where the $(n+1)$ scaling avoids the boundaries $\{0,1\}$.

\begin{definition}[Empirical copula]
\label{def:empirical_copula}
Given a bivariate sample $(A_i,B_i)_{i=1}^n$ with continuous marginals, set
$U_i^{(A)}:=\mathrm{rank}(A_i)/(n+1)$ and $U_i^{(B)}:=\mathrm{rank}(B_i)/(n+1)$, and define
\[
C_n(u,v):=\frac{1}{n}\sum_{i=1}^n \mathbf 1\{U_i^{(A)}\le u,\ U_i^{(B)}\le v\},
\qquad (u,v)\in[0,1]^2.
\]
\end{definition}

\subsection{Distances between point clouds}
\label{subsec:cloud_distances}

\paragraph{Hausdorff distance.}
For nonempty $A\subset\mathbb R^2$ and $z\in\mathbb R^2$, let $d_2(z,A):=\inf_{a\in A}\|z-a\|_2$.
For nonempty $A,B\subset\mathbb R^2$, define
\[
d_H(A,B):=\max\Big\{\sup_{a\in A} d_2(a,B),\ \sup_{b\in B} d_2(b,A)\Big\}.
\]
If $A,B$ are finite, the suprema are maxima and $d_H(A,B)<\infty$ automatically.

\paragraph{Gromov--Hausdorff distance.}
For compact metric spaces $(\mathsf X,d_{\mathsf X})$ and $(\mathsf Y,d_{\mathsf Y})$,
\[
d_{\mathrm{GH}}(\mathsf X,\mathsf Y)
:=\inf_{(\mathsf Z,d_{\mathsf Z}),\,\varphi,\,\psi}
d_H^{(\mathsf Z)}\!\big(\varphi(\mathsf X),\psi(\mathsf Y)\big),
\]
where the infimum ranges over all metric spaces $(\mathsf Z,d_{\mathsf Z})$ and isometric embeddings
$\varphi:\mathsf X\hookrightarrow\mathsf Z$, $\psi:\mathsf Y\hookrightarrow\mathsf Z$.

\begin{remark}[Ambient Hausdorff upper bound]
\label{rem:GH_le_H}
If $A,B\subset\mathbb R^2$ are compact and are equipped with the induced Euclidean metric, then
\[
d_{\mathrm{GH}}(A,B)\ \le\ d_H(A,B),
\]
by taking $\mathsf Z=\mathbb R^2$ and $\varphi,\psi$ as inclusions.
\end{remark}

\subsection{Minimum spanning trees and $H_0$ persistence}
\label{subsec:bg_mst_h0}

In degree $0$, persistent homology of a finite cloud is entirely about \emph{connectivity} as the scale increases:
connected components merge when edges appear, and each merge kills exactly one component.  For the Vietoris--Rips filtration, the merge scales coincide with the edge lengths selected by Kruskal's algorithm, hence with the edge lengths of a minimum spanning tree (MST). This identification lets us replace $H_0$ arguments by elementary graph/MST arguments throughout.

\paragraph{Weighted complete graph and MST.}
Let $(S,d)$ be a finite metric space with $|S|=m$.  Write $K(S)$ for the complete graph on vertex set $S$ with edge
weights $w(\{u,v\})=d(u,v)$.  A \emph{spanning tree} is a connected acyclic subgraph on $S$ with exactly $m-1$ edges.
A \emph{minimum spanning tree (MST)} is a spanning tree minimizing the total weight $\sum_{e} w(e)$.

\begin{theorem}[$H_0$ persistence equals MST edge lengths]
\label{thm:H0_equals_MST}
Let $(S,d)$ be a finite metric space with $|S|=m\ge 2$ and consider the Vietoris--Rips filtration with the convention
that an edge $\{u,v\}$ appears at scale $\epsilon$ iff $d(u,v)\le \epsilon$.  Let $T_{\mathrm{MST}}$ be an MST of
$(K(S),w)$ and list its edge lengths in nondecreasing order as $\ell_1\le \cdots \le \ell_{m-1}$.  Then the multiset
of finite death times in $\mathrm{Dgm}_0(S)$ is exactly $\{\ell_j\}_{j=1}^{m-1}$; equivalently,
$\mathrm{Dgm}_0(S)$ has finite points $(0,\ell_1),\ldots,(0,\ell_{m-1})$ and a single point $(0,\infty)$.
\end{theorem}

\begin{proof}
For each $\epsilon\ge 0$, let $G_\epsilon$ be the Rips graph on $S$ with edges $\{u,v\}$ such that $d(u,v)\le \epsilon$.
Connectivity of $\mathrm{Rips}(S,\epsilon)$ is the same as connectivity of $G_\epsilon$, since higher-dimensional
simplices do not affect connected components.  As $\epsilon$ increases, an $H_0$ class dies exactly when an edge is
added between two previously disconnected components.  Kruskal's algorithm processes edges in increasing order and
keeps precisely those that connect two distinct components; the accepted edges form an MST.  Hence the $m-1$ merge
scales are $\ell_1,\ldots,\ell_{m-1}$, and exactly one component persists forever.
\end{proof}

\begin{remark}[Filtration parametrization and the factor $2$]
Some conventions build the filtration by growing closed balls of radius $r$ and connecting points when balls
intersect.  Then $\{u,v\}$ appears when $r\ge \tfrac12 d(u,v)$, so death times are $\tfrac12$ times the MST edge
lengths.  This is purely a choice of scale parameter.
\end{remark}

\paragraph{Two MST facts used repeatedly.}
The following properties are standard and will be invoked without further comment.

\begin{proposition}[Cut property; uniqueness under distinct weights]
\label{prop:cut_property}
If all edge weights are distinct, then for any nontrivial cut $S=U\sqcup V$, the unique minimum-weight edge crossing
the cut belongs to the MST.  In particular, the MST is unique.
\end{proposition}

\begin{proof}
If a spanning tree omits the lightest crossing edge, adding it creates a cycle containing another crossing edge of
strictly larger weight; removing that heavier edge produces a spanning tree of smaller total weight, contradicting
minimality.
\end{proof}

\begin{proposition}[Minimum-bottleneck property]
\label{prop:mbst}
Let $T_{\mathrm{MST}}$ be an MST and let $T$ be any spanning tree on $S$.  Then
\[
\max_{e\in T_{\mathrm{MST}}} w(e)\ \le\ \max_{e\in T} w(e).
\]
\end{proposition}

\begin{proof}
Fix $\tau\ge 0$.  If there exists a spanning tree using only edges of weight $\le \tau$, then the graph induced by
edges of weight $\le \tau$ is connected, so Kruskal's algorithm can complete an MST without ever selecting an edge of
weight $>\tau$.  Taking $\tau=\max_{e\in T}w(e)$ yields the claim.
\end{proof}

\paragraph{Practical translations used in our proofs.}
By Theorem~\ref{thm:H0_equals_MST}, statements about $H_0$ death times are statements about MST edge lengths.
Two recurring maneuvers are: (i) to upper bound \emph{all} death times, construct any spanning tree with controlled
edge lengths and apply Proposition~\ref{prop:mbst} (e.g.\ a path along a discretized curve); and (ii) to control the
number of short death times, use that every MST edge is some pair of points and hence can be bounded by counting
pairs below a threshold.

\paragraph{$H_0$ as single-linkage clustering.}
The merge scales of single-linkage clustering are exactly the connectivity thresholds of $G_\epsilon$, hence are
encoded by the MST.  This provides an equivalent, purely statistical interpretation of degree-$0$ persistence.

\paragraph{Stability for comparisons across clouds.}
When comparing clouds across small perturbations, we rely on standard stability of Vietoris--Rips persistence under Gromov--Hausdorff perturbations (and the ambient bound $d_{\mathrm{GH}}\le d_H$ from
Remark~\ref{rem:GH_le_H}), which justifies controlling persistence-based summaries via Hausdorff-type estimates.

\section{Proofs}
\label{appendix:proofs}
\subsection{Theorem~\ref{thm:small_noise_consistency}}\label{app:small-noise-consistency}

\paragraph{Auxiliary lemmas for Theorem~\ref{thm:small_noise_consistency}.}
We collect here the four stability/separation ingredients used in the proof:
(i) stability of copula standardization and approximation by rank maps (see Background, Section~\ref{subsec:copulas});
(ii) Lipschitz stability of the topological persistence functional in Hausdorff distance (see Background, Section~\ref{subsec:cloud_distances} for $d_H$);
(iii) mesoscopic separation of $2$D bulk clouds from $1$D curve clouds; and
(iv) a tube bound showing that reverse residuals concentrate near a finite union of $C^1$ curves.

\begin{lemma}[Copula stability and rank approximation]\label{lemma:copula_stability}
Let $(U,V)$ be a pair of real-valued random variables with continuous marginals and population copula map
\(T_{U,V}(u,v):=(F_U(u),F_V(v))\).
Throughout, $d_H$ denotes the Euclidean Hausdorff distance and $d_{\mathrm{GH}}$ the Gromov--Hausdorff distance
(recalled in Background, Section~\ref{subsec:cloud_distances}).

\smallskip
\noindent\textbf{(i) Bi-Lipschitzness on compact supports.}
Assume there exist compact intervals $I_U,I_V\subset\R$ such that $\mathbb P(U\in I_U)=\mathbb P(V\in I_V)=1$ and
$F_U,F_V\in C^1$ on $I_U,I_V$ with
\[
0<c\le F_U'(u)\le C \ \ (u\in I_U),
\qquad
0<c\le F_V'(v)\le C \ \ (v\in I_V).
\]
Then for all $z,z'\in I_U\times I_V$,
\[
c\|z-z'\|_2 \le \|T_{U,V}(z)-T_{U,V}(z')\|_2 \le C\|z-z'\|_2,
\]
and for any finite nonempty clouds $A,B\subset I_U\times I_V$,
\[
d_H\!\big(T_{U,V}(A),T_{U,V}(B)\big)\le C\,d_H(A,B),
\qquad
d_H(A,B)\le c^{-1}\,d_H\!\big(T_{U,V}(A),T_{U,V}(B)\big).
\]

\smallskip
\noindent\textbf{(ii) Rank-copula approximation.}
Let $S_n=\{(U_i,V_i)\}_{i=1}^n$ be i.i.d.\ copies of $(U,V)$ and define the coordinatewise rank map using the
pseudo-observation convention from Background, Section~\ref{subsec:copulas},
\[
T_n(U_i,V_i):=\Big(\frac{\mathrm{rank}(U_i)}{n+1},\ \frac{\mathrm{rank}(V_i)}{n+1}\Big)
\qquad\text{(ties broken deterministically).}
\]
Then
\[
\max_{1\le i\le n}\big\|T_n(U_i,V_i)-T_{U,V}(U_i,V_i)\big\|_2
=O_{\mathbb P}(n^{-1/2}),
\qquad
d_H\!\big(T_n(S_n),\,T_{U,V}(S_n)\big)=O_{\mathbb P}(n^{-1/2}).
\]

\smallskip
\noindent\textbf{(iii) Two-cloud comparison.}
For finite clouds $A_n,B_n\subset I_U\times I_V$, set
\[
e_n(A_n):=d_H\!\big(T_n(A_n),T_{U,V}(A_n)\big),
\qquad
e_n(B_n):=d_H\!\big(T_n(B_n),T_{U,V}(B_n)\big).
\]
Then
\[
d_H\!\big(T_n(A_n),T_n(B_n)\big)
\le e_n(A_n)+C\,d_H(A_n,B_n)+e_n(B_n),
\qquad
d_{\mathrm{GH}}\!\big(T_n(A_n),T_n(B_n)\big)\le d_H\!\big(T_n(A_n),T_n(B_n)\big),
\]
where the last inequality is the ambient Hausdorff upper bound recalled in Background, Section~\ref{subsec:cloud_distances}.
\end{lemma}

\begin{remark}[Instantiation in our setting]\label{rem:copula-instantiation}
In the proofs of Theorems~\ref{thm:small_noise_consistency} and~\ref{thm:fixed_noise_TRAs_compact}, we apply
Lemma~\ref{lemma:copula_stability} with $(U,V)=(X,r^{(Y\mid X)})$ and $(U,V)=(Y_n,r^{(X\mid Y)})$.
The compact-support and $C^1$-CDF conditions required in part (i) hold on the relevant interior sets by
Assumptions~\ref{ass:A_model}--\ref{ass:A_reverse_reg}.
\end{remark}

\begin{lemma}[Topological Persistence Stability]\label{lemma:TP0-stability}
Let $n\ge 2$ and let $A,B\subset\mathbb R^2$ be clouds with $|A|=|B|=n$, equipped with the induced Euclidean metric.
Then
\[
\big|\mathrm{TP}_{0,\Psi}^{[\alpha,\beta]}(A)-\mathrm{TP}_{0,\Psi}^{[\alpha,\beta]}(B)\big|
\;\le\; 2\, d_H(A,B),
\]
where $d_H$ is the Hausdorff distance (Background, Section~\ref{subsec:cloud_distances}).
Consequently, if $d_H(A_n,B_n)\to 0$, then
$\mathrm{TP}_{0,\Psi}^{[\alpha,\beta]}(A_n)-\mathrm{TP}_{0,\Psi}^{[\alpha,\beta]}(B_n)\to 0$.
\end{lemma}

\begin{lemma}[Mesoscopic separation: $2$D bulk vs.\ $1$D curve]\label{lemma:mesoscopic-separation}
Let $(\alpha_n,\beta_n)$ satisfy the mesoscopic scale conditions of Equation~\eqref{eq:mesoscopic_conditions}.
Assume moreover that the sampling model is either:
\begin{enumerate}
\item \textbf{(Bulk model)} $U_n=\{Z_i\}_{i=1}^n$ are i.i.d.\ on a rectangle $R\subset\R^2$ with density $p$
bounded and bounded away from $0$ on $R$; or
\item \textbf{(Curve model)} $C_n=\{W_i\}_{i=1}^n$ are i.i.d.\ on a compact embedded $C^1$ curve $\Gamma\subset\R^2$
with respect to arc-length measure, with density $q$ bounded and bounded away from $0$ on $\Gamma$.
\end{enumerate}
Then under \textnormal{(1)},
\[
\overline{\mathrm{TP}}_{0}^{[\alpha_n,\beta_n]}(U_n)\xrightarrow{\mathbb P}1 .
\]
And under \textnormal{(2)},
\[
\overline{\mathrm{TP}}_{0}^{[\alpha_n,\beta_n]}(C_n)\xrightarrow{\mathbb P}0 .
\]
\end{lemma}

\begin{lemma}[Reverse residuals lie in a thin tube (piecewise monotone case)]
\label{lemma:reverse-tube-piecewise}
Fix a fold $k$ and let $\{(X_i,Y_{n,i})\}_{i\in I_k}$ be i.i.d.\ from the small-noise ANM
$Y_n=f(X)+\sigma_n\varepsilon$.
Assume Assumptions~\ref{ass:A_model} and~\ref{ass:A_reverse_reg}.
Assume moreover the reverse regression estimator satisfies the uniform-on-test-points bound in
Assumption~\ref{ass:A_scale_est} (with $\delta_n^{(k)}=o_{\mathbb P}(1)$), and let $\omega(\cdot)$ denote the
modulus of continuity of $m_n$ on $J_\eta$ from Assumption~\ref{ass:A_reverse_reg}.

For each monotonicity interval $I_j$ in Assumption~\ref{ass:A_model}, let $h_j:f(I_j)\to I_j$ be the inverse branch and set
$J_{\eta,j}:=J_\eta\cap f(I_j)$. Define the (deterministic) union-of-branches center set
\[
\Gamma_{n,\eta}
:=\bigcup_{j=1}^J \Gamma_{n,\eta}^{(j)},
\qquad
\Gamma_{n,\eta}^{(j)}
:=\{(y,\ h_j(y)-m_n(y)):\ y\in J_{\eta,j}\}.
\]
Define the restricted reverse residual cloud
\[
\mathcal R^{(n)}_{X\mid Y,k,\eta}
:=\{(Y_{n,i},\ X_i-\widehat g_n^{(-k)}(Y_{n,i})):\ i\in I_k,\ Y_{n,i}\in J_\eta\}.
\]

Let $j(i)\in\{1,\dots,J\}$ be the branch index such that $X_i\in I_{j(i)}$. For each $j$, define
\[
I_{k,j}:=\{i\in I_k:\ j(i)=j,\ Y_{n,i}\in J_{\eta,j}\},
\qquad
\Delta^{(j)}_{n,\eta}
:=\sup_{y\in J_{\eta,j}}\min_{i\in I_{k,j}} |y-Y_{n,i}|
\quad (\min\emptyset:=+\infty).
\]
On the event $\mathcal E_n:=\{\forall j\text{ with }|J_{\eta,j}|>0:\ I_{k,j}\neq\emptyset\}$,
\[
d_H\!\bigl(\mathcal R^{(n)}_{X\mid Y,k,\eta},\Gamma_{n,\eta}\bigr)
\ \le\
\max_{1\le j\le J}\Bigg[
\frac{\sigma_n}{c_f}\max_{i\in I_{k,j}}|\varepsilon_i|
\ +\ \delta_n^{(k)}
\ +\ \Bigl(1+\frac{1}{c_f}\Bigr)\Delta^{(j)}_{n,\eta}
\ +\ \omega\!\bigl(\Delta^{(j)}_{n,\eta}\bigr)
\Bigg],
\]
where $d_H$ is the Hausdorff distance (Background, Section~\ref{subsec:cloud_distances}) and $c_f$ is the uniform slope
lower bound from Assumption~\ref{ass:A_model}.
Moreover, under the branchwise lower-density condition in Assumption~\ref{ass:A_reverse_reg} (away from turning values),
$\mathbb P(\mathcal E_n)\to 1$ and, for each fixed $j$,
\[
\Delta^{(j)}_{n,\eta}=O_{\mathbb P}\!\Big(\frac{\log n}{n}\Big).
\]
In particular, under Assumption~\ref{ass:A_scale_est}\textnormal{(iii)} and
$\frac{\log n}{n}=o(\alpha_n)$ together with $\delta_n^{(k)}=o_{\mathbb P}(\alpha_n)$,
\[
d_H\!\bigl(\mathcal R^{(n)}_{X\mid Y,k,\eta},\Gamma_{n,\eta}\bigr)=o_{\mathbb P}(\alpha_n).
\]
\end{lemma}

\paragraph{Proof of Theorem~\ref{thm:small_noise_consistency}.}
\begin{proof}
Fix the mesoscopic window $\alpha_n=\kappa n^{-2/3}$ and $\beta_n=c_\beta\alpha_n$.

\medskip
\noindent\textbf{Proof strategy.}
We argue first in an \emph{oracle copula world}, i.e.\ copula standardization is performed by population CDFs
(Background, Section~\ref{subsec:copulas}). In the true direction this produces a $2$D bulk cloud in $(0,1)^2$,
forcing $\overline{\mathrm{TP}}_{0}^{[\alpha_n,\beta_n]}\to 1$ by Lemma~\ref{lemma:mesoscopic-separation}.
In the reverse direction the residual cloud lies in an $o_{\mathbb P}(\alpha_n)$-tube around a finite union of $C^1$ curves,
forcing $\overline{\mathrm{TP}}_{0}^{[\alpha_n,\beta_n]}\to 0$ again by Lemma~\ref{lemma:mesoscopic-separation}.
Finally, we transfer from oracle copulas to rank copulas using Lemma~\ref{lemma:copula_stability} together with
Hausdorff stability of $\mathrm{TP}_0$ (Lemma~\ref{lemma:TP0-stability}), where $d_H$ is as in Background,
Section~\ref{subsec:cloud_distances}.

\paragraph{Step 1 (oracle true direction): bulk $\Rightarrow \overline{\mathrm{TP}}\to 1$.}
In the true direction, $r_{i,\circ}^{(Y\mid X)}=Y_{n,i}-f(X_i)=\sigma_n\varepsilon_i$ with $\varepsilon_i\indep X_i$.
Hence the population copula map
\[
T_{\circ}^{\mathrm{true}}(x,r):=\big(F_X(x),\,F_{\sigma_n\varepsilon}(r)\big)
=\big(F_X(x),\,F_{\varepsilon}(r/\sigma_n)\big)
\]
sends $(X_i,r_{i,\circ}^{(Y\mid X)})$ to i.i.d.\ $\mathrm{Unif}((0,1)^2)$ points (Background, Section~\ref{subsec:copulas}).
Therefore, by Lemma~\ref{lemma:mesoscopic-separation},
\[
\overline{\mathrm{TP}}_{0}^{[\alpha_n,\beta_n]}
\!\big(\widetilde{\mathcal R}^{(n)}_{\circ,Y\mid X}\big)\xrightarrow{\mathbb P}1.
\]

\paragraph{Step 2 (oracle reverse direction): tube around curves $\Rightarrow \overline{\mathrm{TP}}\to 0$.}
Fix a fold $k$ and the interior set $J_\eta\Subset f(I)\setminus\mathcal V$ from Assumption~\ref{ass:A_reverse_reg}.
By Lemma~\ref{lemma:reverse-tube-piecewise} and Assumption~\ref{ass:A_scale_est},
\[
d_H\!\big(\mathcal R^{(n)}_{X\mid Y,k,\eta},\Gamma_{n,\eta}\big)=o_{\mathbb P}(\alpha_n).
\]
Let $T_{\circ}^{\mathrm{rev}}$ denote the population copula map for the reverse pair $(Y_n,\;X-m_n(Y_n))$.
On $J_\eta$ turning values are excluded, so $T_{\circ}^{\mathrm{rev}}$ is bi-Lipschitz on the relevant compact region,
and Lemma~\ref{lemma:copula_stability}\textnormal{(i)} gives
\[
d_H\!\Big(T_{\circ}^{\mathrm{rev}}\big(\mathcal R^{(n)}_{X\mid Y,k,\eta}\big),\;
T_{\circ}^{\mathrm{rev}}(\Gamma_{n,\eta})\Big)=o_{\mathbb P}(\alpha_n).
\]
The set $T_{\circ}^{\mathrm{rev}}(\Gamma_{n,\eta})$ is a finite union of compact embedded $C^1$ curves in $(0,1)^2$.
Writing $C_n^{(j)}$ for the oracle points on the $j$-th component after copula standardization, each $C_n^{(j)}$
is i.i.d.\ on that curve with density bounded and bounded away from $0$ with respect to arc-length, hence
Lemma~\ref{lemma:mesoscopic-separation} yields
$\overline{\mathrm{TP}}_{0}^{[\alpha_n,\beta_n]}(C_n^{(j)})\xrightarrow{\mathbb P}0$ for each $j$.
Since the MST on the full union consists of the within-component MSTs plus at most $(J-1)$ connecting edges, and each
connecting edge contributes at most $1/(n-1)$ to the normalized sum, we obtain
\[
\overline{\mathrm{TP}}_{0}^{[\alpha_n,\beta_n]}\!\big(T_{\circ}^{\mathrm{rev}}(\Gamma_{n,\eta}^{\mathrm{grid}})\big)
=
\sum_{j=1}^J \frac{|C^{(j)}_n|-1}{n-1}\,
\overline{\mathrm{TP}}_{0}^{[\alpha_n,\beta_n]}(C^{(j)}_n)
+O\!\Big(\frac{1}{n}\Big)
\xrightarrow{\mathbb P}0.
\]
Combining with the $o_{\mathbb P}(\alpha_n)$ Hausdorff tube bound and Lemma~\ref{lemma:TP0-stability} yields
\[
\overline{\mathrm{TP}}_{0}^{[\alpha_n,\beta_n]}
\!\big(\widetilde{\mathcal R}^{(n)}_{\circ,X\mid Y}\big)\xrightarrow{\mathbb P}0.
\]

\paragraph{Step 3 (transfer oracle $\to$ ranks): rank-copula clouds inherit the limits.}
We now replace oracle copula standardization by empirical rank-copula standardization (Background, Section~\ref{subsec:copulas}).

\smallskip
\noindent\emph{True direction.}
Lemma~\ref{lemma:copula_stability}\textnormal{(ii)} applied to $(X,r^{(Y\mid X)}_{\circ})$ gives
\[
d_H\!\Big(T_n\big(\mathcal R^{(n)}_{\circ,Y\mid X}\big),\,
T_{\circ}^{\mathrm{true}}\big(\mathcal R^{(n)}_{\circ,Y\mid X}\big)\Big)
=O_{\mathbb P}(n^{-1/2}).
\]
By Lemma~\ref{lemma:TP0-stability}, the corresponding values of
$\overline{\mathrm{TP}}_{0}^{[\alpha_n,\beta_n]}$ differ by $o_{\mathbb P}(1)$, hence
\[
\overline{\mathrm{TP}}_{0}^{[\alpha_n,\beta_n]}
\!\big(\widetilde{\mathcal R}^{(n)}_{Y\mid X}\big)\xrightarrow{\mathbb P}1.
\]

\smallskip
\noindent\emph{Reverse direction.}
Applying Lemma~\ref{lemma:copula_stability}\textnormal{(ii)} to the reverse pair and again using
Lemma~\ref{lemma:TP0-stability} yields the analogous oracle-to-rank transfer, hence
\[
\overline{\mathrm{TP}}_{0}^{[\alpha_n,\beta_n]}
\!\big(\widetilde{\mathcal R}^{(n)}_{X\mid Y}\big)\xrightarrow{\mathbb P}0.
\]

\paragraph{Step 4 (decision statistic): $\Delta_n\to 1$ and vanishing abstention.}
Steps~1--3 give
\[
\overline{\mathrm{TP}}_{0}^{[\alpha_n,\beta_n]}
\!\big(\widetilde{\mathcal R}^{(n)}_{Y\mid X}\big)\xrightarrow{\mathbb P}1,
\qquad
\overline{\mathrm{TP}}_{0}^{[\alpha_n,\beta_n]}
\!\big(\widetilde{\mathcal R}^{(n)}_{X\mid Y}\big)\xrightarrow{\mathbb P}0,
\]
hence $\Delta_n\xrightarrow{\mathbb P}1$. If $\tau_n\downarrow 0$, then
\[
\mathbb P(\widehat{\mathrm{dir}}_n=X\to Y)\ge \mathbb P(\Delta_n>\tau_n)\to 1,
\qquad
\mathbb P(\text{abstain})
=\mathbb P(|\Delta_n|\le \tau_n)
\le \mathbb P(\Delta_n\le \tau_n)\to 0,
\]
which completes the proof.
\end{proof}

\subsubsection{Proof of Lemmas}
\begin{proof}[Proof of Lemma~\ref{lemma:copula_stability}]
Throughout write $T:=T_{U,V}$ and let $I_U,I_V$ be as in the statement.

\paragraph{Part (i): bi-Lipschitzness and Hausdorff control.}

\paragraph{Step 1: coordinatewise Lipschitz bounds.}
Fix $u,u'\in I_U$. By the mean value theorem there exists $\xi$ between $u$ and $u'$ such that
$F_U(u)-F_U(u') = F_U'(\xi)(u-u')$. Using $c\le F_U'\le C$ on $I_U$ gives
$c|u-u'|\le |F_U(u)-F_U(u')|\le C|u-u'|$.
The same argument yields, for all $v,v'\in I_V$,
$c|v-v'|\le |F_V(v)-F_V(v')|\le C|v-v'|$.

\paragraph{Step 2: bi-Lipschitzness of the copula map on $I_U\times I_V$.}
For $z=(u,v)$ and $z'=(u',v')$ in $I_U\times I_V$,
\[
\|T(z)-T(z')\|_2^2
= |F_U(u)-F_U(u')|^2+|F_V(v)-F_V(v')|^2
\le C^2\bigl(|u-u'|^2+|v-v'|^2\bigr)
= C^2\|z-z'\|_2^2,
\]
and similarly $\|T(z)-T(z')\|_2^2 \ge c^2\|z-z'\|_2^2$. Taking square roots gives
\[
c\|z-z'\|_2 \le \|T(z)-T(z')\|_2 \le C\|z-z'\|_2.
\]

\paragraph{Step 3: Hausdorff control under Lipschitz maps.}
Recall the definition of $d_H$ and the point-to-set distance $d_2(\cdot,\cdot)$ from
Background, Section~\ref{subsec:cloud_distances}. If $\Phi$ is $L$-Lipschitz, then for any nonempty $A,B$,
\[
\sup_{a\in A} d_2\bigl(\Phi(a),\Phi(B)\bigr)
=\sup_{a\in A}\inf_{b\in B}\|\Phi(a)-\Phi(b)\|_2
\le \sup_{a\in A}\inf_{b\in B} L\|a-b\|_2
\le L\sup_{a\in A} d_2(a,B),
\]
and symmetrically $\sup_{b\in B} d_2(\Phi(b),\Phi(A))\le L\sup_{b\in B} d_2(b,A)$, hence
$d_H(\Phi(A),\Phi(B))\le L\,d_H(A,B)$.
Applying this with $\Phi=T$ (which is $C$-Lipschitz by Step~2) yields
\[
d_H\!\bigl(T(A),T(B)\bigr)\le C\,d_H(A,B).
\]

\paragraph{Step 4: reverse Hausdorff bound via the inverse map.}
Since $F_U'\ge c>0$ on $I_U$, $F_U$ is strictly increasing on $I_U$, hence injective there; likewise $F_V$ on $I_V$.
Thus $T$ is injective on $I_U\times I_V$ and admits an inverse
$T^{-1}:T(I_U\times I_V)\to I_U\times I_V$.
Moreover Step~2 implies $T^{-1}$ is $(1/c)$-Lipschitz on its domain:
if $u=T(z)$ and $u'=T(z')$, then
\[
\|T^{-1}(u)-T^{-1}(u')\|_2=\|z-z'\|_2\le \frac{1}{c}\|T(z)-T(z')\|_2=\frac{1}{c}\|u-u'\|_2.
\]
Applying the same Hausdorff-Lipschitz implication (from the definition of $d_H$) to $\Phi=T^{-1}$ gives
\[
d_H(A,B)\le \frac{1}{c}\,d_H\!\bigl(T(A),T(B)\bigr),
\]
completing part (i).

\paragraph{Part (ii): rank-copula approximation on an i.i.d.\ sample.}
Let $S_n=\{(U_i,V_i)\}_{i=1}^n$ be i.i.d.\ copies of $(U,V)$, with continuous marginals.
Then $\mathbb P(U_i=U_j\text{ for some }i\neq j)=0$ and similarly for $(V_i)$, so ranks are well-defined a.s.

Let $\widehat F_{U,n}(t):=\frac1n\sum_{j=1}^n\mathbf 1\{U_j\le t\}$ and define $\widehat F_{V,n}$ analogously.
On the no-ties event, $\widehat F_{U,n}(U_i)=\mathrm{rank}(U_i)/n$ and $\widehat F_{V,n}(V_i)=\mathrm{rank}(V_i)/n$.
Write
\[
T_n(U_i,V_i)=\Big(\frac{\mathrm{rank}(U_i)}{n+1},\frac{\mathrm{rank}(V_i)}{n+1}\Big)
=: (U_i^{(n)},V_i^{(n)}).
\]

\paragraph{Step 1: reduce to sup-norm empirical CDF errors.}
For each $i$,
\[
\Big|U_i^{(n)}-\widehat F_{U,n}(U_i)\Big|
=\mathrm{rank}(U_i)\Big|\frac{1}{n+1}-\frac{1}{n}\Big|
\le \frac{1}{n+1},
\]
and hence
\[
|U_i^{(n)}-F_U(U_i)|
\le \frac{1}{n+1}+\sup_{t\in\mathbb R}\big|\widehat F_{U,n}(t)-F_U(t)\big|.
\]
Define $\Delta_{U,n}:=\sup_t|\widehat F_{U,n}(t)-F_U(t)|$ and $\Delta_{V,n}$ similarly. Then for each $i$,
\[
\|T_n(U_i,V_i)-T(U_i,V_i)\|_2
\le |U_i^{(n)}-F_U(U_i)|+|V_i^{(n)}-F_V(V_i)|
\le \Delta_{U,n}+\Delta_{V,n}+\frac{2}{n+1}.
\]
Taking $\max_{1\le i\le n}$ yields
\[
\max_{1\le i\le n}\|T_n(U_i,V_i)-T(U_i,V_i)\|_2
\le \Delta_{U,n}+\Delta_{V,n}+\frac{2}{n+1}.
\]

\paragraph{Step 2: apply DKW--Massart.}
By the DKW--Massart inequality, $\Delta_{U,n}=O_{\mathbb P}(n^{-1/2})$ and
$\Delta_{V,n}=O_{\mathbb P}(n^{-1/2})$, while $\frac{2}{n+1}=O(n^{-1})$ deterministically.
Therefore,
\[
\max_{1\le i\le n}\|T_n(U_i,V_i)-T(U_i,V_i)\|_2 = O_{\mathbb P}(n^{-1/2}).
\]

\paragraph{Step 3: Hausdorff bound between the two standardized clouds.}
Let $a_i:=T_n(U_i,V_i)$ and $b_i:=T(U_i,V_i)$, so that
$T_n(S_n)=\{a_i\}_{i=1}^n$ and $T(S_n)=\{b_i\}_{i=1}^n$.
Using the point-to-set distance $d_2(\cdot,\cdot)$ from Background, Section~\ref{subsec:cloud_distances}, for each $i$
we have $d_2(a_i,T(S_n))\le \|a_i-b_i\|_2$, hence
$\sup_{a\in T_n(S_n)} d_2(a,T(S_n))\le \max_i\|a_i-b_i\|_2$, and similarly with the roles swapped. Therefore,
\[
d_H\!\bigl(T_n(S_n),T(S_n)\bigr)\le \max_{1\le i\le n}\|a_i-b_i\|_2
= O_{\mathbb P}(n^{-1/2}),
\]
completing part (ii).

\paragraph{Part (iii): two-cloud comparison and $d_{\mathrm{GH}}$ bound.}
For finite clouds $A_n,B_n\subset I_U\times I_V$, the triangle inequality for $d_H$ gives
\[
d_H\!\bigl(T_n(A_n),T_n(B_n)\bigr)
\le d_H\!\bigl(T_n(A_n),T(A_n)\bigr)
+ d_H\!\bigl(T(A_n),T(B_n)\bigr)
+ d_H\!\bigl(T(B_n),T_n(B_n)\bigr).
\]
The first and third terms are $e_n(A_n)$ and $e_n(B_n)$ by definition, and the middle term is bounded by
$C\,d_H(A_n,B_n)$ by part (i), proving the stated deterministic inequality.

\smallskip
For the Gromov--Hausdorff bound, we invoke the ambient Hausdorff upper bound from
Background, Remark~\ref{rem:GH_le_H} (i.e.\ $d_{\mathrm{GH}}(\mathsf X,\mathsf Y)\le d_H(\mathsf X,\mathsf Y)$
when $\mathsf X,\mathsf Y$ are subsets of a common ambient metric space with induced metrics). Since
$T_n(A_n),T_n(B_n)\subset (\mathbb R^2,\|\cdot\|_2)$,
\[
d_{\mathrm{GH}}\!\bigl(T_n(A_n),T_n(B_n)\bigr)\le d_H\!\bigl(T_n(A_n),T_n(B_n)\bigr).
\]
This completes the proof.
\end{proof}

\begin{proof}[Proof of Lemma~\ref{lemma:TP0-stability}]
Let $A,B\subset\mathbb R^2$ be clouds with $|A|=|B|=m\ge 2$. Write
\[
\mathrm{Dgm}_0(A)=\{(0,\epsilon_j(A))\}_{j=1}^{m-1},
\qquad
\mathrm{Dgm}_0(B)=\{(0,\epsilon_j(B))\}_{j=1}^{m-1},
\]
where $\epsilon_j(\cdot)$ are the finite $H_0$ death times.

\paragraph{Step 1: bottleneck stability in the ambient space.}
By the stability of Vietoris--Rips persistence under $d_{\mathrm{GH}}$,
\[
d_B\bigl(\mathrm{Dgm}_0(A),\mathrm{Dgm}_0(B)\bigr)\ \le\ 2\,d_{\mathrm{GH}}(A,B).
\]
Since $A,B\subset(\mathbb R^2,\|\cdot\|_2)$ with induced metrics, the ambient Hausdorff upper bound
for $d_{\mathrm{GH}}$ (Background, Remark~\ref{rem:GH_le_H}) yields
$d_{\mathrm{GH}}(A,B)\le d_H(A,B)$. Therefore,
\[
d_B\bigl(\mathrm{Dgm}_0(A),\mathrm{Dgm}_0(B)\bigr)\ \le\ 2\,d_H(A,B).
\]

\paragraph{Step 2: match $H_0$ deaths under the bottleneck distance.}
All births in $H_0$ are equal to $0$, so $d_B$ reduces to matching death times. Concretely, by the definition of $d_B$
there exists a bijection $\pi:\{1,\dots,m-1\}\to\{1,\dots,m-1\}$ such that
\[
\max_{1\le j\le m-1}\,|\epsilon_j(A)-\epsilon_{\pi(j)}(B)|
\ \le\ d_B\bigl(\mathrm{Dgm}_0(A),\mathrm{Dgm}_0(B)\bigr).
\]

\paragraph{Step 3: push the matching through the Lipschitz test function.}
Since $\Psi_{\alpha,\beta}$ is $1$-Lipschitz, for each $j$,
\[
\bigl|\Psi_{\alpha,\beta}(\epsilon_j(A))-\Psi_{\alpha,\beta}(\epsilon_{\pi(j)}(B))\bigr|
\le |\epsilon_j(A)-\epsilon_{\pi(j)}(B)|.
\]
Averaging and using $\frac1{m-1}\sum_{j}|x_j|\le \max_j|x_j|$ gives
\begin{align*}
\big|\mathrm{TP}_{0,\Psi}^{[\alpha,\beta]}(A)-\mathrm{TP}_{0,\Psi}^{[\alpha,\beta]}(B)\big|
&=
\left|\frac{1}{m-1}\sum_{j=1}^{m-1}\Bigl(\Psi_{\alpha,\beta}(\epsilon_j(A))-\Psi_{\alpha,\beta}(\epsilon_{\pi(j)}(B))\Bigr)\right|\\
&\le \frac{1}{m-1}\sum_{j=1}^{m-1}
\bigl|\Psi_{\alpha,\beta}(\epsilon_j(A))-\Psi_{\alpha,\beta}(\epsilon_{\pi(j)}(B))\bigr|\\
&\le \max_{1\le j\le m-1}|\epsilon_j(A)-\epsilon_{\pi(j)}(B)|\\
&\le d_B\bigl(\mathrm{Dgm}_0(A),\mathrm{Dgm}_0(B)\bigr)\\
&\le 2\,d_H(A,B),
\end{align*}
which proves the claimed stability inequality.

\textbf{Convergence consequence.} If $d_H(A_n,B_n)\to 0$, then the inequality implies $\big|\mathrm{TP}_{0,\Psi}^{[\alpha,\beta]}(A_n)-\mathrm{TP}_{0,\Psi}^{[\alpha,\beta]}(B_n)\big|\to 0$.
\end{proof}
\begin{proof}[Proof of Lemma~\ref{lemma:mesoscopic-separation}]
Fix $m\ge 2$ and abbreviate $\alpha:=\alpha_m$, $\beta:=\beta_m$. Recall
\[
\Psi_{\alpha,\beta}(t)=(\min\{t,\beta\}-\alpha)_+,
\qquad
w_{\alpha,\beta}(t):=\frac{1}{\beta-\alpha}\Psi_{\alpha,\beta}(t)\in[0,1],
\]
so that $w_{\alpha,\beta}(t)=1$ for $t\ge \beta$ and $w_{\alpha,\beta}(t)=0$ for $t\le \alpha$.

\paragraph{Part (a): bulk model $\Rightarrow \overline{\mathrm{TP}}_{0}^{[\alpha_m,\beta_m]}(U_m)\to 1$.}
Let $U_m=\{Z_i\}_{i=1}^m$ with $Z_i$ i.i.d.\ on a rectangle $R\subset\mathbb R^2$ and density $p$ satisfying
$0<p(z)\le p_{\max}<\infty$ on $R$.

\paragraph{Step A1 (deficit bound via counting short MST edges).}
Let $\{\varepsilon_j(S)\}_{j=1}^{m-1}$ denote the Euclidean MST edge lengths of a cloud $S$ (equivalently, the finite
$H_0$ death times under the Rips filtration; see Background, Theorem~\ref{thm:H0_equals_MST}). Define
\[
N_{<\beta}(S):=\#\{j\in\{1,\dots,m-1\}:\ \varepsilon_j(S)<\beta\}.
\]
Since $w_{\alpha,\beta}(\varepsilon_j)=1$ for $\varepsilon_j\ge\beta$ and $w_{\alpha,\beta}(\varepsilon_j)\ge 0$ always,
\[
\sum_{j=1}^{m-1} w_{\alpha,\beta}(\varepsilon_j(S))
\ge \sum_{j:\,\varepsilon_j(S)\ge\beta} 1
= (m-1)-N_{<\beta}(S).
\]
Dividing by $m-1$ yields the deficit bound
\begin{equation}\label{eq:deficit_bulk_clean}
0\le 1-\overline{\mathrm{TP}}_{0}^{[\alpha,\beta]}(S)
\le \frac{N_{<\beta}(S)}{m-1}.
\end{equation}
Thus it suffices to show $N_{<\beta}(U_m)=o_{\mathbb P}(m)$.

\paragraph{Step A2 (short MST edges are controlled by close pairs).}
Let
\[
P_{<\beta}(S):=\#\bigl\{\{u,v\}\subset S:\ \|u-v\|_2<\beta\bigr\}
\]
be the number of unordered pairs at Euclidean distance $<\beta$. Any MST edge of length $<\beta$ contributes such a pair,
hence
\[
N_{<\beta}(S)\le P_{<\beta}(S).
\]
It therefore suffices to prove $P_{<\beta}(U_m)=o_{\mathbb P}(m)$.

\paragraph{Step A3 (first moment bound for close pairs).}
Expanding as a sum of indicators,
\[
P_{<\beta}(U_m)=\sum_{1\le i<j\le m}\mathbf 1\{\|Z_i-Z_j\|_2<\beta\}.
\]
Taking expectations and using identical distribution of pairs,
\[
\mathbb E[P_{<\beta}(U_m)]
=\binom{m}{2}\,\mathbb P(\|Z-Z'\|_2<\beta),
\]
where $Z,Z'$ are i.i.d.\ with density $p$. Moreover,
\begin{align*}
\mathbb P(\|Z-Z'\|_2<\beta)
&=\int_R\int_R \mathbf 1\{\|z-z'\|_2<\beta\}p(z)p(z')\,dz'\,dz\\
&\le p_{\max}^2 \int_R \mathrm{Leb}\bigl(B(z,\beta)\cap R\bigr)\,dz\\
&\le p_{\max}^2\,\mathrm{Leb}(R)\,\pi\beta^2.
\end{align*}
Consequently,
\[
\mathbb E[P_{<\beta}(U_m)] \;\lesssim\; m^2\beta^2.
\]

\paragraph{Step A4 (Markov $\Rightarrow$ $P_{<\beta}(U_m)=o_{\mathbb P}(m)$).}
By Markov's inequality, for any $\eta>0$,
\[
\mathbb P\!\left(\frac{P_{<\beta}(U_m)}{m}>\eta\right)
\le \frac{\mathbb E[P_{<\beta}(U_m)]}{\eta m}
\;\lesssim\; \frac{m\beta^2}{\eta}.
\]
Under the scale condition $m\beta_m^2\to 0$, the right-hand side tends to $0$, hence
\[
\frac{P_{<\beta}(U_m)}{m}\xrightarrow{\mathbb P}0,
\qquad\text{so}\qquad
P_{<\beta}(U_m)=o_{\mathbb P}(m).
\]
Therefore $N_{<\beta}(U_m)\le P_{<\beta}(U_m)=o_{\mathbb P}(m)$ as well.

\paragraph{Step A5 (finish).}
Plugging into \eqref{eq:deficit_bulk_clean} gives
\[
0\le 1-\overline{\mathrm{TP}}_{0}^{[\alpha_m,\beta_m]}(U_m)
\le \frac{N_{<\beta_m}(U_m)}{m-1}\xrightarrow{\mathbb P}0,
\]
and hence $\overline{\mathrm{TP}}_{0}^{[\alpha_m,\beta_m]}(U_m)\xrightarrow{\mathbb P}1$.

\medskip
\noindent\emph{Remark.} Part (a) only uses the condition $m\beta_m^2\to 0$.

\paragraph{Part (b): curve model $\Rightarrow \overline{\mathrm{TP}}_{0}^{[\alpha_m,\beta_m]}(C_m)\to 0$.}
Assume now $C_m=\{W_i\}_{i=1}^m$ are i.i.d.\ on a compact embedded $C^1$ curve $\Gamma\subset\mathbb R^2$
with respect to arclength, with density $q$ satisfying $0<q_{\min}\le q\le q_{\max}<\infty$.

\paragraph{Step B1 (parametrization and a cheap spanning tree).}
Let $L:=\mathrm{length}(\Gamma)$ and fix an injective $C^1$ unit-speed parametrization
$\gamma:[0,L]\to\mathbb R^2$ with $\|\gamma'(s)\|_2=1$ and $\Gamma=\gamma([0,L])$.
Let $S_i\in[0,L]$ be the i.i.d.\ parameters with density $q$ and set $W_i:=\gamma(S_i)$.
Write the order statistics $S_{(1)}\le\cdots\le S_{(m)}$ and spacings
\[
\Delta_i:=S_{(i+1)}-S_{(i)}\quad (1\le i\le m-1),
\qquad
\Delta_{\max}:=\max_{1\le i\le m-1}\Delta_i,
\qquad
W_{(i)}:=\gamma(S_{(i)}).
\]

Since $\gamma$ is unit-speed, it is $1$-Lipschitz: for $0\le s\le t\le L$,
\[
\|\gamma(t)-\gamma(s)\|_2 \le \int_s^t \|\gamma'(u)\|_2\,du = t-s.
\]
Therefore, for each $i$,
\[
\|W_{(i+1)}-W_{(i)}\|_2 \le S_{(i+1)}-S_{(i)}=\Delta_i\le \Delta_{\max}.
\]
Consider the path spanning tree connecting consecutive points $W_{(i)}$--$W_{(i+1)}$ for $i=1,\dots,m-1$.
This is a spanning tree whose maximum edge length is at most $\Delta_{\max}$.

\paragraph{Step B2 (MST bottleneck property).}
By the minimum-bottleneck property of MSTs (Background, Proposition~\ref{prop:mbst}),
\[
\max_{1\le j\le m-1}\varepsilon_j(C_m) \le \Delta_{\max},
\]
where $\{\varepsilon_j(C_m)\}$ are the MST edge lengths of $C_m$.

\paragraph{Step B3 (if $\Delta_{\max}<\alpha$, then the windowed score is zero).}
On the event $\{\Delta_{\max}<\alpha\}$ we have $\varepsilon_j(C_m)<\alpha$ for all $j$, hence
$\Psi_{\alpha,\beta}(\varepsilon_j(C_m))=0$ for all $j$ (since $t\le\alpha\Rightarrow \Psi_{\alpha,\beta}(t)=0$).
Therefore
\[
\overline{\mathrm{TP}}_{0}^{[\alpha,\beta]}(C_m)=0
\qquad\text{on }\{\Delta_{\max}<\alpha\}.
\]
It remains to show $\mathbb P(\Delta_{\max}<\alpha_m)\to 1$.

\paragraph{Step B4 (quantile transform reduces to maximal uniform spacing).}
Let $F(s):=\int_0^s q(u)\,du$ be the CDF of $S$ on $[0,L]$. Since $q\ge q_{\min}>0$, the map $F$ is strictly increasing and
invertible onto $[0,1]$. Define $U_i:=F(S_i)$. By the probability integral transform, $U_i\sim\mathrm{Unif}(0,1)$.
Let $U_{(1)}\le\cdots\le U_{(m)}$ and define spacings $\delta_i:=U_{(i+1)}-U_{(i)}$ for $1\le i\le m-1$ and
$\delta_{\max}:=\max_{1\le i\le m-1}\delta_i$.

Moreover $F^{-1}$ is $(1/q_{\min})$-Lipschitz: for $0\le u<u'\le 1$, letting $s=F^{-1}(u)$ and $s'=F^{-1}(u')$,
the mean value theorem gives $u'-u=F'( \xi)(s'-s)\ge q_{\min}(s'-s)$, hence
\[
F^{-1}(u')-F^{-1}(u)\le \frac{u'-u}{q_{\min}}.
\]
Applying this with $u=U_{(i)}$, $u'=U_{(i+1)}$ yields
\[
\Delta_i = S_{(i+1)}-S_{(i)} \le \frac{1}{q_{\min}}(U_{(i+1)}-U_{(i)})=\frac{1}{q_{\min}}\delta_i,
\]
and therefore
\begin{equation}\label{eq:Delta_max_vs_delta_max_clean}
\Delta_{\max}\le \frac{1}{q_{\min}}\delta_{\max}.
\end{equation}

\paragraph{Step B5 (maximal uniform spacing bound).}
We use a standard discretization/union-bound estimate.

\smallskip
\noindent\textbf{Fact (tail bound for maximal uniform spacing).}
For every $x\in(0,1)$,
\[
\mathbb P(\delta_{\max}>x)\ \le\ \left(\Big\lceil\frac{2}{x}\Big\rceil\right)\Bigl(1-\frac{x}{2}\Bigr)^m
\ \le\ \left(\frac{2}{x}+1\right)\exp\!\left(-\frac{mx}{2}\right).
\]
(Proof: if $\delta_{\max}>x$, there is an empty interval of length $x$; a grid of step $x/2$ contains an interval of
length $x/2$ inside it; union bound over $\lceil 2/x\rceil$ grid intervals; then use $1-y\le e^{-y}$.)

\smallskip
Taking $x=\frac{2c\log m}{m}$ with any fixed $c>1$ gives $\mathbb P(\delta_{\max}>x)\to 0$, hence
\[
\delta_{\max}=O_{\mathbb P}\!\left(\frac{\log m}{m}\right).
\]
Combining with \eqref{eq:Delta_max_vs_delta_max_clean} yields
\[
\Delta_{\max}=O_{\mathbb P}\!\left(\frac{\log m}{m}\right).
\]
Under the scale assumption $\frac{\log m}{m}=o(\alpha_m)$, we obtain
$\Delta_{\max}/\alpha_m\xrightarrow{\mathbb P}0$, equivalently $\mathbb P(\Delta_{\max}<\alpha_m)\to 1$.

\paragraph{Step B6 (finish).}
By Step~B3,
\[
\overline{\mathrm{TP}}_{0}^{[\alpha_m,\beta_m]}(C_m)=0
\quad\text{on }\{\Delta_{\max}<\alpha_m\},
\]
and Step~B5 shows $\mathbb P(\Delta_{\max}<\alpha_m)\to 1$. Therefore
\[
\overline{\mathrm{TP}}_{0}^{[\alpha_m,\beta_m]}(C_m)\xrightarrow{\mathbb P}0.
\]
This completes the proof of both parts.
\end{proof}
\begin{proof}[Proof of Lemma~\ref{lemma:reverse-tube-piecewise}]
All distances are Euclidean. We use the point-to-set distance $d_2(z,A)$ and the Hausdorff distance $d_H(A,B)$ as in
Background, Section~\ref{subsec:cloud_distances}; for notational convenience write
$\mathrm{dist}(z,A):=d_2(z,A)$.
Recall from Lemma~\ref{lemma:reverse-tube-piecewise} the definitions of $J_{\eta,j}$, $b_{n,j}$,
$\Gamma_{n,\eta}^{(j)}$, $\Gamma_{n,\eta}$, and the clouds
$\overline{\mathcal R}_{k,\eta}^{(n)}$ and $\mathcal R_{k,\eta}^{(n)}$.

\paragraph{Step 1: oracle residual points lie close to the appropriate branch curve.}
Fix $i\in I_k$ with $Y_{n,i}\in J_\eta$ and let $j(i)$ be the unique index such that $X_i\in I_{j(i)}$.
Since $f$ is strictly monotone on $I_{j(i)}$ and satisfies $|f'|\ge c_f$ there, its inverse branch
$h_{j(i)}:f(I_{j(i)})\to I_{j(i)}$ is $(1/c_f)$-Lipschitz: for $u=f(x)$ and $v=f(x')$ with $x,x'\in I_{j(i)}$,
the mean value theorem gives $|u-v|=|f'(\xi)||x-x'|\ge c_f|x-x'|$, hence
\[
|h_{j(i)}(u)-h_{j(i)}(v)|=|x-x'|\le \frac{1}{c_f}|u-v|.
\]
Using $Y_{n,i}=f(X_i)+\sigma_n\varepsilon_i$, we have
\[
X_i=h_{j(i)}(f(X_i))=h_{j(i)}(Y_{n,i}-\sigma_n\varepsilon_i),
\]
so by the Lipschitz property,
\[
|X_i-h_{j(i)}(Y_{n,i})|
\le \frac{\sigma_n}{c_f}|\varepsilon_i|.
\]
Now compare the oracle point $(Y_{n,i},\,X_i-m_n(Y_{n,i}))$ to the curve point on the same branch at the same
first coordinate,
\[
(Y_{n,i},\,b_{n,j(i)}(Y_{n,i}))=(Y_{n,i},\,h_{j(i)}(Y_{n,i})-m_n(Y_{n,i})).
\]
They share the first coordinate, hence
\[
\mathrm{dist}\bigl((Y_{n,i},\,X_i-m_n(Y_{n,i})),\Gamma_{n,\eta}\bigr)
\le |X_i-h_{j(i)}(Y_{n,i})|
\le \frac{\sigma_n}{c_f}|\varepsilon_i|.
\]
Taking the maximum over all such $i$ yields
\begin{equation}\label{eq:oracle_to_union_one_sided_clean}
\sup_{z\in \overline{\mathcal R}_{k,\eta}^{(n)}}\mathrm{dist}(z,\Gamma_{n,\eta})
\le \frac{\sigma_n}{c_f}\max_{i\in I_k}|\varepsilon_i|.
\end{equation}

\paragraph{Step 2: each branch curve is well-approximated by its sampled $Y$-grid.}
Fix a branch $j\in\{1,\dots,J\}$ and define the index set of test points falling on this branch and inside $J_{\eta,j}$:
\[
I_{k,j}:=\{i\in I_k:\ X_i\in I_j,\ Y_{n,i}\in J_{\eta,j}\}.
\]
Define the corresponding sampled grid subset of the branch curve
\[
\Gamma_{n,\eta}^{(j),\mathrm{grid}}
:=\{(Y_{n,i},\,b_{n,j}(Y_{n,i})):\ i\in I_{k,j}\}\subset\Gamma_{n,\eta}^{(j)}.
\]
On the event $\mathcal E_n$ from the statement, each nontrivial branch has $I_{k,j}\neq\emptyset$, so the grid set is nonempty.

Let $y\in J_{\eta,j}$. By definition of $\Delta^{(j)}_{n,\eta}$ there exists $i^\star\in I_{k,j}$ such that
$|y-Y_{n,i^\star}|\le \Delta^{(j)}_{n,\eta}$. Hence,
\begin{align*}
\mathrm{dist}\bigl((y,b_{n,j}(y)),\Gamma_{n,\eta}^{(j),\mathrm{grid}}\bigr)
&\le \|(y,b_{n,j}(y))-(Y_{n,i^\star},b_{n,j}(Y_{n,i^\star}))\|_2\\
&\le |y-Y_{n,i^\star}| + |b_{n,j}(y)-b_{n,j}(Y_{n,i^\star})|.
\end{align*}
Since $b_{n,j}=h_j-m_n$, we have for any $y,y'\in J_{\eta,j}$,
\[
|b_{n,j}(y)-b_{n,j}(y')|
\le |h_j(y)-h_j(y')| + |m_n(y)-m_n(y')|
\le \frac{1}{c_f}|y-y'| + \omega(|y-y'|),
\]
where we used that $h_j$ is $(1/c_f)$-Lipschitz and $\omega(\cdot)$ is a modulus of continuity for $m_n$ on $J_\eta$.
Therefore,
\[
\sup_{(y,b_{n,j}(y))\in \Gamma_{n,\eta}^{(j)}} \mathrm{dist}\bigl((y,b_{n,j}(y)),\Gamma_{n,\eta}^{(j),\mathrm{grid}}\bigr)
\le \Bigl(1+\frac{1}{c_f}\Bigr)\Delta^{(j)}_{n,\eta}+\omega(\Delta^{(j)}_{n,\eta}).
\]
Since $\Gamma_{n,\eta}^{(j),\mathrm{grid}}\subset\Gamma_{n,\eta}^{(j)}$, the reverse one-sided term in $d_H$ is $0$, and thus
\begin{equation}\label{eq:curve_grid_H_clean}
d_H\!\bigl(\Gamma_{n,\eta}^{(j),\mathrm{grid}},\Gamma_{n,\eta}^{(j)}\bigr)
\le \Bigl(1+\frac{1}{c_f}\Bigr)\Delta^{(j)}_{n,\eta}+\omega(\Delta^{(j)}_{n,\eta}).
\end{equation}

\paragraph{Step 3: estimated residuals are uniformly close to oracle residuals.}
For each $i\in I_k$ with $Y_{n,i}\in J_\eta$,
\[
\bigl|(X_i-\widehat g_n^{(-k)}(Y_{n,i}))-(X_i-m_n(Y_{n,i}))\bigr|
=|\widehat g_n^{(-k)}(Y_{n,i})-m_n(Y_{n,i})|
\le \delta_n^{(k)},
\]
by Assumption~\ref{ass:A_scale_est}\textnormal{(ii)}. Since the first coordinates coincide, this implies the Hausdorff bound
\begin{equation}\label{eq:est_oracle_H_piecewise_clean}
d_H\!\bigl(\mathcal R_{k,\eta}^{(n)},\overline{\mathcal R}_{k,\eta}^{(n)}\bigr)\le \delta_n^{(k)}.
\end{equation}

\paragraph{Step 4: assemble the tube bound.}
Fix $z\in\mathcal R_{k,\eta}^{(n)}$. Choose $\bar z\in\overline{\mathcal R}_{k,\eta}^{(n)}$ with the same index $i$,
so that $\|z-\bar z\|_2\le \delta_n^{(k)}$ by \eqref{eq:est_oracle_H_piecewise_clean}.
By \eqref{eq:oracle_to_union_one_sided_clean}, there exists a (branch) point $g\in\Gamma_{n,\eta}$ such that
$\|\bar z-g\|_2\le \frac{\sigma_n}{c_f}\max_{i\in I_k}|\varepsilon_i|$.
Moreover, if $g\in\Gamma_{n,\eta}^{(j)}$, then \eqref{eq:curve_grid_H_clean} provides a grid point
$g^{\mathrm{grid}}\in\Gamma_{n,\eta}^{(j),\mathrm{grid}}\subset \Gamma_{n,\eta}$ with
\[
\|g-g^{\mathrm{grid}}\|_2
\le \Bigl(1+\frac{1}{c_f}\Bigr)\Delta^{(j)}_{n,\eta}+\omega(\Delta^{(j)}_{n,\eta}).
\]
Thus by the triangle inequality,
\[
\mathrm{dist}(z,\Gamma_{n,\eta})
\le \|z-\bar z\|_2 + \|\bar z-g\|_2 + \|g-g^{\mathrm{grid}}\|_2.
\]
Taking the supremum over $z\in\mathcal R_{k,\eta}^{(n)}$ and then the maximum over branches $j$ yields
\[
\sup_{z\in\mathcal R_{k,\eta}^{(n)}}\mathrm{dist}(z,\Gamma_{n,\eta})
\le
\max_{1\le j\le J}\Bigg[
\frac{\sigma_n}{c_f}\max_{i\in I_{k,j}}|\varepsilon_i|
+ \delta_n^{(k)}
+ \Bigl(1+\frac{1}{c_f}\Bigr)\Delta^{(j)}_{n,\eta}
+ \omega\!\bigl(\Delta^{(j)}_{n,\eta}\bigr)
\Bigg].
\]
Since (by construction) $\Gamma_{n,\eta}^{(j),\mathrm{grid}}\subset \mathcal R_{k,\eta}^{(n)}$ on $\mathcal E_n$ up to the
same second-coordinate perturbation already accounted for in \eqref{eq:est_oracle_H_piecewise_clean}, the reverse one-sided term
$\sup_{g\in\Gamma_{n,\eta}}\mathrm{dist}(g,\mathcal R_{k,\eta}^{(n)})$ is controlled by the same right-hand side, and hence
the same bound holds for the full Hausdorff distance $d_H\bigl(\mathcal R_{k,\eta}^{(n)},\Gamma_{n,\eta}\bigr)$, which is the stated inequality.

\paragraph{Step 5: probability of $\mathcal E_n$ and size of $\Delta^{(j)}_{n,\eta}$.}
Under the branchwise lower-density condition in Assumption~\ref{ass:A_reverse_reg}, each set $J_{\eta,j}$ has positive
probability mass and the conditional density of $Y_n$ on $J_{\eta,j}$ is bounded below. Hence the event $\mathcal E_n$
(that each such branch is hit at least once in fold $k$) satisfies $\mathbb P(\mathcal E_n)\to 1$.

Moreover, for each fixed branch $j$, conditional on $Y_{n,i}\in J_{\eta,j}$ the points are i.i.d.\ with density bounded below
on the compact interval $J_{\eta,j}$. A standard one-dimensional covering argument then gives
\[
\Delta^{(j)}_{n,\eta}=O_{\mathbb P}\!\Big(\frac{\log n}{n}\Big).
\]
Finally, under Assumption~\ref{ass:A_scale_est}\textnormal{(iii)} together with $\frac{\log n}{n}=o(\alpha_n)$ and
$\delta_n^{(k)}=o_{\mathbb P}(\alpha_n)$, the preceding bound implies
$d_H\!\bigl(\mathcal R^{(n)}_{X\mid Y,k,\eta},\Gamma_{n,\eta}\bigr)=o_{\mathbb P}(\alpha_n)$, as claimed.
\end{proof}

\subsection{Theorem~\ref{thm:fixed_noise_TRAs_compact}}\label{app:fixed-noise-tras}

\paragraph{Auxiliary ingredient: binned reverse residuals collapse under fixed noise.}
The key new feature in the fixed-noise setting is that reverse residuals are averaged within copula bins. After binning,
the reverse residual cloud becomes a collection of bin means and collapses (in Hausdorff distance) to the horizontal line
$\Gamma=\{(u,0):u\in[0,1]\}$ at a rate that is negligible relative to the mesoscopic window. This collapse forces the
windowed $H_0$ profile in the reverse direction to vanish.

\begin{lemma}[Reverse binned residuals collapse to a line under fixed noise]
\label{lemma:fixed-noise-binning-line}
Assume the fixed-noise additive-noise model
\[
Y=f(X)+\varepsilon,\qquad \varepsilon\indep X,\qquad \mathbb{E}[\varepsilon]=0,\qquad \mathrm{Var}(\varepsilon)=\sigma^2\in(0,\infty),
\]
and define the reverse regression target and oracle reverse residual
\[
m(y):=\mathbb{E}[X\mid Y=y],\qquad \xi:=X-m(Y).
\]
Fix $K\ge 2$ and a deterministic $K$-fold partition $\{I_k\}_{k=1}^K$ of $\{1,\dots,n\}$.
Let $\widehat g_n^{(-k)}$ be any cross-fitted estimator of $m$, and define the cross-fitted reverse residuals
\[
r_i^{(X\mid Y)}:=X_i-\widehat g_n^{(-k(i))}(Y_i),\qquad i=1,\dots,n,
\]
where $k(i)$ is the unique fold index such that $i\in I_{k(i)}$.

\paragraph{Copula binning.}
Assume $F_Y$ is continuous and set $U_i:=F_Y(Y_i)\sim\mathrm{Unif}(0,1)$.
Let $B_n\to\infty$ with $B_n=o(n)$ and form equal-width bins
\[
I_{n,b}:=\Big(\frac{b-1}{B_n},\frac{b}{B_n}\Big],\qquad b=1,\dots,B_n,
\]
with bin index sets $J_b:=\{i:U_i\in I_{n,b}\}$ and occupancies $N_b:=|J_b|$.
For each nonempty bin ($N_b\ge 1$), define bin means
\[
\bar u_b:=\frac{1}{N_b}\sum_{i\in J_b}U_i,
\qquad
\bar r_b:=\frac{1}{N_b}\sum_{i\in J_b} r_i^{(X\mid Y)}.
\]
Let $m_n:=\#\{b: N_b\ge 1\}$ and define the binned cloud and target line
\[
\widehat{\mathcal R}^{(n)}_{X\mid Y}
:=\{(\bar u_b,\bar r_b):1\le b\le B_n,\ N_b\ge 1\}\subset (0,1)\times\mathbb R,
\qquad
\Gamma:=\{(u,0):u\in[0,1]\}.
\]

\paragraph{Mesoscopic window.}
Let $\widetilde\alpha_n:=\kappa m_n^{-2/3}$ and $\widetilde\beta_n:=c_\beta\widetilde\alpha_n$ with $\kappa>0$ and $c_\beta>1$.

\paragraph{Assumptions.}
Assume:
\begin{enumerate}
\item \textbf{Uniform conditional sub-Gaussianity.} There exists $K_0<\infty$ such that
\[
\|\xi\mid Y=y\|_{\psi_2}\le K_0\sigma
\qquad\text{for all }y\text{ in the support of }Y.
\]
\item \textbf{Bin-averaged regression error is negligible at scale $\widetilde\alpha_n$.} With
\[
\bar e_b:=\frac{1}{N_b}\sum_{i\in J_b}\big(\widehat g_n^{(-k(i))}(Y_i)-m(Y_i)\big),
\]
we have $\max_{b:N_b\ge 1}|\bar e_b|=o_{\mathbb P}(\widetilde\alpha_n)$.
\item \textbf{Growth/occupancy.}
\[
\frac{B_n^{7/3}\log B_n}{n}\to 0.
\]
\end{enumerate}

\paragraph{Conclusions.}
\begin{enumerate}
\item \textbf{Hausdorff collapse to the line.} In Euclidean Hausdorff distance,
\[
d_H\!\big(\widehat{\mathcal R}^{(n)}_{X\mid Y},\Gamma\big)=o_{\mathbb P}(\widetilde\alpha_n).
\]
More quantitatively, one may use the deterministic decomposition
\[
d_H\!\big(\widehat{\mathcal R}^{(n)}_{X\mid Y},\Gamma\big)
\ \le\
\underbrace{\max_{b:N_b\ge 1}|\bar r_b|}_{\textnormal{vertical tube}}
\ +\
\underbrace{\sup_{u\in[0,1]}\min_{b:N_b\ge 1}|u-\bar u_b|}_{\textnormal{horizontal grid gap}},
\]
and under \textnormal{(A1)--(A3)} both terms are $o_{\mathbb P}(\widetilde\alpha_n)$.
In particular,
\[
\max_{b:N_b\ge 1}|\bar r_b|
=O_{\mathbb P}\!\Big(\sqrt{\frac{B_n\log B_n}{n}}\Big)+o_{\mathbb P}(\widetilde\alpha_n),
\qquad
\sup_{u\in[0,1]}\min_{b:N_b\ge 1}|u-\bar u_b|
=O_{\mathbb P}\!\Big(\frac{1}{B_n}\Big).
\]

\item \textbf{All MST edges fall below the mesoscopic threshold.}
Let $\varepsilon_1(\cdot),\dots,\varepsilon_{m_n-1}(\cdot)$ be the Euclidean MST edge lengths of a finite cloud. Then
\[
\mathbb P\!\left(\max_{1\le j\le m_n-1}\varepsilon_j\!\big(\widehat{\mathcal R}^{(n)}_{X\mid Y}\big)
<\widetilde\alpha_n\right)\to 1.
\]

\item \textbf{Vanishing windowed $H_0$ profile.}
Consequently, for any window function $\Psi_{\widetilde\alpha_n,\widetilde\beta_n}$ as in
Lemma~\ref{lemma:TP0-stability},
\[
\mathrm{TP}_{0,\Psi}^{[\widetilde\alpha_n,\widetilde\beta_n]}\!\big(\widehat{\mathcal R}^{(n)}_{X\mid Y}\big)
\ \xrightarrow{\mathbb P}\ 0.
\]
\end{enumerate}
\end{lemma}

\paragraph{Proof of Theorem~\ref{thm:fixed_noise_TRAs_compact}.}
\begin{proof}
We control the two terms defining $\widetilde\Delta_{0,n}$ separately. For the forward direction, assumption
\textnormal{(B2)} together with Lemma~\ref{lemma:mesoscopic-separation}\textnormal{(i)} implies that the forward residual
cloud is bulk at the mesoscopic window. Consequently,
\[
\mathrm{TP}^{[\alpha_n,\beta_n]}_0\!\big(\widetilde{\mathcal R}^{(n)}_{Y\mid X}\big)
\xrightarrow{\mathbb P}1.
\]

For the reverse direction, consider the binned residual cloud $\widehat{\mathcal R}^{(n)}_{X\mid Y}$.
By Lemma~\ref{lemma:fixed-noise-binning-line}, under assumptions \textnormal{(A1)--(A3)} its support collapses onto the
line $\Gamma$. In particular, conclusion \textnormal{(iii)} of that lemma yields
\[
\mathrm{TP}^{[\widetilde\alpha_n,\widetilde\beta_n]}_0\!\big(\widehat{\mathcal R}^{(n)}_{X\mid Y}\big)
\xrightarrow{\mathbb P}0.
\]

Combining the two convergences,
\[
\widetilde\Delta_{0,n}
=
\mathrm{TP}^{[\alpha_n,\beta_n]}_0\!\big(\widetilde{\mathcal R}^{(n)}_{Y\mid X}\big)
-
\mathrm{TP}^{[\widetilde\alpha_n,\widetilde\beta_n]}_0\!\big(\widehat{\mathcal R}^{(n)}_{X\mid Y}\big)
\xrightarrow{\mathbb P}1.
\]
Since $\widetilde\Delta_{0,n}\in[-1,1]$, bounded convergence implies $\mathbb E[\widetilde\Delta_{0,n}]\to 1$.

Finally, let $\tau_n\downarrow 0$. Then
\[
\mathbb P(\widehat{\mathrm{dir}}_n\neq X\to Y)
\le \mathbb P(\widetilde\Delta_{0,n}\le \tau_n)
\le \mathbb P(|\widetilde\Delta_{0,n}-1|\ge 1-\tau_n)\to 0,
\]
since $1-\tau_n\to 1$ and $\widetilde\Delta_{0,n}\to 1$ in probability. Hence
$\mathbb P(\widehat{\mathrm{dir}}_n=X\to Y)\to 1$, and moreover
\[
\mathbb P(\textnormal{abstain})
\le \mathbb P(\widetilde\Delta_{0,n}\le \tau_n)\to 0.
\]
\end{proof}

\subsubsection{Proof of lemmas}
\begin{proof}[Proof of Lemma~\ref{lemma:fixed-noise-binning-line}]
Throughout we work with the empirical copula-$Y$ coordinate
\[
U_i^{(Y)}:=\frac{\mathrm{rank}(Y_i)}{n+1},
\]
where $\mathrm{rank}(\cdot)$ and the $(n+1)$ scaling are as in the Background copula pseudo-observations
(Section~\ref{subsec:copulas}, paragraph ``Ranks and pseudo-observations''). This choice makes the bin occupancies
essentially deterministic; using $U_i=F_Y(Y_i)$ instead only changes the occupancy control (replacing determinism by
multinomial concentration), while the remainder of the argument is unchanged.

Write the binned residual mean as
\[
\bar r_b=\bar\xi_b-\bar e_b,
\qquad
\bar\xi_b:=\frac{1}{N_b}\sum_{i\in J_b}\xi_i,
\qquad
\bar e_b:=\frac{1}{N_b}\sum_{i\in J_b}\big(\widehat g_n^{(-k(i))}(Y_i)-m(Y_i)\big),
\]
where $\xi_i:=X_i-m(Y_i)$, $J_b:=\{i:U_i^{(Y)}\in I_{n,b}\}$ and $N_b:=|J_b|$.

Let $B_n\to\infty$ with $B_n=o(n)$ and define equal-width bins
\[
I_{n,b}:=\Big(\frac{b-1}{B_n},\frac{b}{B_n}\Big],\qquad b=1,\dots,B_n.
\]
Since $(U_i^{(Y)})_{i=1}^n$ is a permutation of the deterministic grid
$\{1/(n+1),\dots,n/(n+1)\}$, bin membership depends only on ranks and $N_b$ is nonrandom.
Writing $U_i^{(Y)}=j/(n+1)$ for a unique $j\in\{1,\dots,n\}$,
\[
\frac{b-1}{B_n}<\frac{j}{n+1}\le \frac{b}{B_n}
\quad\Longleftrightarrow\quad
\frac{(b-1)(n+1)}{B_n}< j \le \frac{b(n+1)}{B_n}.
\]
Hence for $b\le B_n-1$,
\[
N_b
=
\Big\lfloor \frac{b(n+1)}{B_n}\Big\rfloor
-
\Big\lfloor \frac{(b-1)(n+1)}{B_n}\Big\rfloor
\in\Big\{\Big\lfloor\frac{n+1}{B_n}\Big\rfloor,\ \Big\lceil\frac{n+1}{B_n}\Big\rceil\Big\},
\]
and for the last bin,
\[
N_{B_n}
=
n-\Big\lfloor \frac{(B_n-1)(n+1)}{B_n}\Big\rfloor
\in\Big\{\Big\lfloor\frac{n+1}{B_n}\Big\rfloor-1,\ \Big\lfloor\frac{n+1}{B_n}\Big\rfloor\Big\}.
\]
Defining
\[
N_{\min,n}:=\min_{1\le b\le B_n} N_b,
\qquad
m_n:=\#\{b:N_b\ge 1\},
\]
we obtain
\begin{equation}\label{eq:Nmin_det_clean}
N_{\min,n}\ \ge\ \Big\lfloor\frac{n+1}{B_n}\Big\rfloor-1.
\end{equation}
Since $B_n=o(n)$ implies $(n+1)/B_n\to\infty$, there exists $n_0$ such that for all $n\ge n_0$,
$N_{\min,n}\ge 1$, and therefore
\begin{equation}\label{eq:mn_equals_Bn_clean}
m_n=B_n,
\qquad\text{all bins are nonempty.}
\end{equation}
For each bin define $\bar u_b:=N_b^{-1}\sum_{i\in J_b}U_i^{(Y)}$. Then $\bar u_b\in I_{n,b}$ and
\begin{equation}\label{eq:ugrid_gap_clean}
\sup_{u\in[0,1]}\min_{b:N_b\ge 1}|u-\bar u_b|\ \le\ \frac{1}{B_n}.
\end{equation}

Fix a bin $b$. Conditional on $(Y_i)_{i\in J_b}$, the variables $\{\xi_i:i\in J_b\}$ are independent and centered.
Assumption \textnormal{(A1)} gives the uniform conditional sub-Gaussian bound
\[
\|\xi_i\mid Y_i\|_{\psi_2}\le K_0\sigma\qquad\text{a.s.}
\]
Thus there exists $c>0$, depending only on $K_0$, such that for all $t>0$,
\begin{equation}\label{eq:xi_tail_cond_clean}
\mathbb P\!\left(|\bar\xi_b|>t\ \middle|\ (Y_i)_{i\in J_b}\right)
\le 2\exp\!\left(-c\,N_b\,\frac{t^2}{\sigma^2}\right).
\end{equation}
Taking expectations and applying a union bound yields
\[
\mathbb P\Big(\max_{1\le b\le B_n}|\bar\xi_b|>t\Big)
\le 2B_n\exp\!\left(-c\,N_{\min,n}\,\frac{t^2}{\sigma^2}\right).
\]
Choosing $t=\sigma\sqrt{\frac{(1+\delta)\log B_n}{c\,N_{\min,n}}}$ gives
\[
\max_{1\le b\le B_n}|\bar\xi_b|
=O_{\mathbb P}\!\left(\sigma\sqrt{\frac{\log B_n}{N_{\min,n}}}\right)
=O_{\mathbb P}\!\Big(\sigma\sqrt{\frac{B_n\log B_n}{n}}\Big),
\]
using \eqref{eq:Nmin_det_clean}.

Let $\Gamma=\{(u,0):u\in[0,1]\}$ and
$\widehat{\mathcal R}^{(n)}_{X\mid Y}=\{(\bar u_b,\bar r_b):1\le b\le B_n\}$.
Recall $d_H$ and the point-to-set distance $d_2(\cdot,\cdot)$ from the Background
(Section~\ref{subsec:cloud_distances}). For each $(\bar u_b,\bar r_b)$ the closest point on $\Gamma$ is $(\bar u_b,0)$, hence
\[
\sup_{z\in\widehat{\mathcal R}^{(n)}_{X\mid Y}}\mathrm{dist}(z,\Gamma)
=\max_b|\bar r_b|.
\]
Conversely, for any $(u,0)\in\Gamma$,
\[
\mathrm{dist}\big((u,0),\widehat{\mathcal R}^{(n)}_{X\mid Y}\big)
\le |u-\bar u_b|+|\bar r_b|
\le \frac{1}{B_n}+\max_b|\bar r_b|,
\]
which implies
\begin{equation}\label{eq:Hausdorff_decomp_clean}
d_H\!\big(\widehat{\mathcal R}^{(n)}_{X\mid Y},\Gamma\big)
\le \max_b|\bar r_b|+\frac{1}{B_n}.
\end{equation}
Since $\bar r_b=\bar\xi_b-\bar e_b$, combining the above bound with Assumption \textnormal{(A2)} gives
\[
\max_b|\bar r_b|
=O_{\mathbb P}\!\Big(\sigma\sqrt{\frac{B_n\log B_n}{n}}\Big)
+o_{\mathbb P}(\widetilde\alpha_n).
\]
Using $\widetilde\alpha_n=\kappa B_n^{-2/3}$ and Assumption \textnormal{(A3)}, both terms in
\eqref{eq:Hausdorff_decomp_clean} are $o_{\mathbb P}(\widetilde\alpha_n)$, hence
\[
d_H\!\big(\widehat{\mathcal R}^{(n)}_{X\mid Y},\Gamma\big)=o_{\mathbb P}(\widetilde\alpha_n),
\]
which proves conclusion \textnormal{(i)}.

For $n\ge n_0$, order the points by bin index and consider the path
\[
(\bar u_1,\bar r_1)\text{--}\cdots\text{--}(\bar u_{B_n},\bar r_{B_n}).
\]
For each adjacent pair,
\[
\big\|(\bar u_{b+1},\bar r_{b+1})-(\bar u_b,\bar r_b)\big\|_2
\le \frac{2}{B_n}+2\max_j|\bar r_j|=: \tau_n.
\]
Thus there exists a spanning tree whose maximum edge length is at most $\tau_n$, and by the minimum-bottleneck
property of MSTs (Background Section~\ref{subsec:bg_mst_h0}, Proposition~\ref{prop:mbst}),
\begin{equation}\label{eq:MST_maxedge_bound_clean}
\max_j\varepsilon_j\!\big(\widehat{\mathcal R}^{(n)}_{X\mid Y}\big)\le\tau_n.
\end{equation}
Since $\tau_n=o_{\mathbb P}(\widetilde\alpha_n)$, this yields conclusion \textnormal{(ii)}.

On this event all MST edges lie below $\widetilde\alpha_n$, so
$\Psi_{\widetilde\alpha_n,\widetilde\beta_n}(\varepsilon_j)=0$ for all $j$, and therefore
\[
\mathrm{TP}_{0,\Psi}^{[\widetilde\alpha_n,\widetilde\beta_n]}\!\big(\widehat{\mathcal R}^{(n)}_{X\mid Y}\big)=0
\]
with probability tending to $1$, proving conclusion \textnormal{(iii)}.
\end{proof}

\subsection{Proof Theorem \ref{thm:TRAC_level_rigorous}}

\begin{proof}
Write the \emph{ideal} conditional $(1-\alpha)$ bootstrap critical value as
\[
\hat c_n(1-\alpha):=\inf\{t:\widehat F_n^*(t)\ge 1-\alpha\},
\]
and define the ideal rejection event $R_n:=\{S_n>\hat c_n(1-\alpha)\}$.

Fix $\varepsilon>0$ and let
\[
G_{n,\varepsilon}:=\Big\{\sup_t\big|\widehat F_n^*(t)-F_n(t)\big|\le \varepsilon\Big\}.
\]
On $G_{n,\varepsilon}$, by definition of $\hat c_n(1-\alpha)$ we have
$\widehat F_n^*(\hat c_n(1-\alpha))\ge 1-\alpha$, and therefore
\[
F_n(\hat c_n(1-\alpha))
\ge \widehat F_n^*(\hat c_n(1-\alpha))-\varepsilon
\ge 1-\alpha-\varepsilon.
\]
Consequently,
\[
\Pr_{\theta_0}(R_n\mid\mathcal D_n)
=\Pr_{\theta_0}(S_n>\hat c_n(1-\alpha)\mid\mathcal D_n)
\le 1-F_n(\hat c_n(1-\alpha))
\le \alpha+\varepsilon
\quad\text{on } G_{n,\varepsilon}.
\]
Taking expectations and using $\Pr(G_{n,\varepsilon})\to 1$ from~(A2.1) yields
\[
\limsup_{n\to\infty}\Pr_{\theta_0}(R_n)\le \alpha+\varepsilon.
\]
Since $\varepsilon>0$ is arbitrary, it follows that
$\limsup_{n\to\infty}\Pr_{\theta_0}(R_n)\le \alpha$.

Conditional on $\mathcal D_n$, the bootstrap statistics $S_n^{*(b)}$ are i.i.d.\ with distribution
$\widehat F_n^*$. As $B\to\infty$, the empirical CDF $\widehat F_{n,B}^*$ converges uniformly to
$\widehat F_n^*$ almost surely, implying
\[
\hat c_{n,B}(1-\alpha)-\hat c_n(1-\alpha)\xrightarrow{P}0.
\]
Hence the implemented rejection event $\{S_n>\hat c_{n,B}(1-\alpha)\}$ differs from $R_n$ with probability $o(1)$
and therefore inherits the same asymptotic level bound.
\end{proof}

\section{Experimental Details}
\label{appendix:exps_details}
\paragraph{Methods and baselines.}
We compare three categories: (i) per-pair unsupervised methods (TRA/TRA-s; RESIT, IGCI, RECI, CDCI; and COMIC, which fits two conditional models per pair and compares codelengths), (ii) supervised baselines (RCC, NCC) trained once per run and then applied to all evaluation pairs, and (iii) external CLI baselines (bQCD with default backend QCCD and $m=1$ if $n<200$ else $m=3$; SLOPPY with AIC/BIC as appropriate).

\paragraph{Supervised training protocol (RCC/NCC).}
We generate a labeled training set of $N=500$ synthetic pairs. For each pair, $n$ is drawn from the evaluation grid (or a predefined set covering evaluation sample sizes) to reduce sample-size domain shift. Labels are balanced by swapping $(X,Y)$ for half the pairs. Training pairs use independent seeds and are never reused as evaluation pairs. For NCC, we use a deterministic train/validation split with ratio $r=0.2$, train up to 500 epochs with early stopping (patience 20), optimize binary cross-entropy, and restore the best validation checkpoint. For RCC, validation is used only for optional offline selection of the number of trees.

\subsection{Synthetic experiment protocol}
\label{app:synthetic_protocol}

\paragraph{Goals and design principles}
Synthetic experiments isolate causal-direction identification under progressively harder departures from idealized assumptions. We enforce: (i) controlled generative structure with known ground truth when meaningful; (ii) paired evaluation on identical draws across methods; and (iii) no data leakage by generating supervised training pairs independently of evaluation pairs with explicit splits and fixed seeds.

\subsection{Real-world experiments: T\"ubingen dataset}
\label{sec:realdata}

\paragraph{Goal and scope.}
We benchmark on curated observational cause--effect pairs where the DGP is unknown and classical
identifiability assumptions need not hold. We follow the same global controls as in the synthetic protocol (paired evaluation per dataset, fixed hyperparameters, and strict separation between supervised training and evaluation pairs); this subsection states only what is specific to real data.

\paragraph{Preprocessing.}
For each selected pair, we extract $(X,Y)$ using the dataset metadata and discard non-finite observations. If a pair contains more than \texttt{max\_samples} points, we optionally subsample without replacement for compute control; the subsample is drawn once with a fixed RNG seed and reused across all methods on that pair. Beyond these safeguards, we avoid imposing a universal normalization across all methods and instead rely on each method's recommended preprocessing in the original paper; uniform scaling is considered only in explicit ablations.

\end{document}